%% file: iclr2027_conference.tex
\documentclass{article} 
\usepackage{iclr2027_conference,times}

\input{math_commands.tex}

\usepackage{hyperref}
\usepackage{url}
\usepackage[utf8]{inputenc} 
\usepackage[T1]{fontenc}    
\usepackage{booktabs}       
\usepackage{amsfonts}       
\usepackage{nicefrac}       
\usepackage{microtype}      
\usepackage{xcolor}         
\usepackage{multirow}
\usepackage{tabularx}
\usepackage{makecell}
\usepackage{amsmath}
\usepackage{amssymb}
\usepackage{amsthm}
\newtheorem{theorem}{Theorem}
\newtheorem{lemma}{Lemma}
\newtheorem{corollary}{Corollary}
\usepackage{graphicx}
\usepackage{makecell}
\usepackage{enumitem}
\usepackage[most]{tcolorbox}
\iclrfinalcopy

\newtcolorbox{takeaway}{
  enhanced,
  colback=blue!3,
  colframe=blue!60!black,
  boxrule=0pt,
  leftrule=2.5pt,
  arc=0pt,
  left=7pt,
  right=7pt,
  top=5pt,
  bottom=5pt,
  before skip=8pt,
  after skip=8pt
}
\title{Where Should Diffusion Enter a Language Model? Geometry-Guided Hidden-State Replacement}

\author{
  \textbf{Injin Kong}\quad
  \textbf{Hyoungjoon Lee}\quad
  \textbf{Yohan Jo}\thanks{Corresponding author.}\\
  Graduate School of Data Science, Seoul National University \\
  \texttt{mtkong77@snu.ac.kr, hjoon721@snu.ac.kr, yohan.jo@snu.ac.kr} \\
}

\begin{document}

\maketitle

\begin{abstract}
Continuous diffusion language models require choosing a representation space in which denoising operates, yet it remains unclear which representations are most compatible with diffusion. We study a design question: where inside a pretrained language model should continuous diffusion operate? We formulate this as a hidden-state interface-selection problem and hypothesize that diffusion-friendly interfaces can be identified from representation geometry. We operationalize this hypothesis using three training-free geometric proxies: local compactness, global stiffness, and effective rank. Across three pretrained backbones, the resulting geometry score predicts diffusion bridgeability beyond layer depth and identifies a non-trivial interior optimum. We then instantiate \textbf{DiHAL}, a Locate-and-Replace framework that replaces the transformer prefix below a selected interface with conditional diffusion while retaining the pretrained suffix and LM head. Under full training, geometry-selected interfaces remain close to validation-loss oracles and outperform worst-layer controls, while the diffusion bridge outperforms a parameter-matched deterministic replacement under matched diagnostic conditions. These results suggest that the representation space in which diffusion operates should be treated as a first-class design variable for continuous diffusion in language models.
\end{abstract}

\section{Introduction}
\label{sec:intro}

Diffusion models, which generate data through iterative denoising, have achieved remarkable success in continuous domains such as image generation \citep{song2021scorebased}. This success has motivated growing interest in diffusion as an alternative to autoregressive language modeling, with both discrete and continuous formulations \citep{li2022diffusionlm, strudel2023selfconditioned, shi2024simplified, nie2025large}. For language generation, iterative denoising offers the potential to refine multiple token representations jointly rather than following a strictly left-to-right generation order. Despite recent advances, continuous diffusion language models still lag behind autoregressive language models in generation quality \citep{jo2025continuous}.

In this work, we focus on a challenge for continuous diffusion in language: the interface between continuous denoising and discrete token generation. Prior work has diffused token embeddings, self-conditioned embeddings, or learned text latents, together with mechanisms that map the resulting representations back to language \citep{li2022diffusionlm, strudel2023selfconditioned, lovelace2023latent, gong2023diffuseqv2, zhang-etal-2025-diffusion}. When token recovery depends directly on the denoised representation, errors in representation space can alter the resulting discrete output \citep{li2022diffusionlm}.

Yet this difficulty does not make continuous representations themselves unsuitable for language generation. Autoregressive language models already compute through a stack of continuous hidden states \citep{NIPS2017_3f5ee243}. At each intermediate layer, the pretrained suffix provides a learned mapping from its hidden states to the output distribution.

This suggests placing continuous diffusion inside the pretrained model rather than at the continuous-to-token interface. Although Transformer hidden states are deterministic given their inputs, reconstructing them under limited capacity and training budgets remains challenging. We investigate diffusion as an iterative reconstruction mechanism leveraging the pretrained suffix for token recovery.

This formulation raises a design question: where should diffusion enter the pretrained transformer? The insertion layer determines both the hidden-state distribution diffusion must model and the suffix that consumes it. Because these distributions vary across depth, their suitability for diffusion may also differ. We therefore view insertion as an interface-selection problem: identifying an internal representation that can be modeled reliably while remaining compatible with the retained suffix.

We hypothesize that this suitability is mirrored in hidden-state geometry. In particular, we seek \emph{diffusion-friendly} representations that permit fast contraction during denoising, remain robust to imperfect score estimation, and concentrate variation in few active directions. Idealized results from Langevin dynamics and concentration theory motivate geometric quantities associated with these properties \citep{villani2009optimal,bakry2014analysis,ledoux2001concentration}. We approximate them by three layer-wise proxies: local compactness, global stiffness, and effective rank, which provide a training-free way to screen candidate insertion layers without training a diffusion model at each depth.

Based on this idea, we develop a \emph{Locate-and-Replace} framework, instantiated as \textbf{DiHAL} (\textbf{Di}ffusion-Transformer \textbf{H}ybrid \textbf{A}rchitecture for \textbf{L}anguage Generation). DiHAL \emph{locates} a promising insertion layer using geometric proxies, then \emph{replaces} the preceding transformer layers with conditional diffusion while retaining the pretrained suffix and LM head. Rather than improving or accelerating autoregressive models, DiHAL serves as a framework for investigating continuous diffusion interfaces. Across three pretrained backbones, our geometry criterion predicts diffusion bridgeability and identifies an interior optimum on a held-out model, outperforming shallow-layer selection even under full training. Diffusion also outperforms a same-architecture deterministic predictor in 13 of 16 layers under a matched fixed-budget protocol. Our contributions are summarized below.

\begin{takeaway}
\begin{description}[
    leftmargin=5.3em,
    labelwidth=4.5em,
    itemsep=0.15em,
    topsep=0em,
    parsep=0em,
    style=multiline
]
\item[\textbf{Interface}]
We formulate continuous diffusion as a hidden-state interface-selection problem.

\item[\textbf{Geometry}]
We introduce geometry-guided interface selection using three theory-motivated proxies---local compactness, global stiffness, and effective rank---validated on synthetic distributions and across backbones.

\item[\textbf{Evidence}]
We instantiate this principle in \textbf{DiHAL} and show that geometry-selected interfaces transfer to full training and compare favorably with alternative continuous-diffusion interfaces under shared training budgets.
\end{description}
\end{takeaway}

\begin{figure}[t]
\centering
\includegraphics[width=0.9\linewidth]{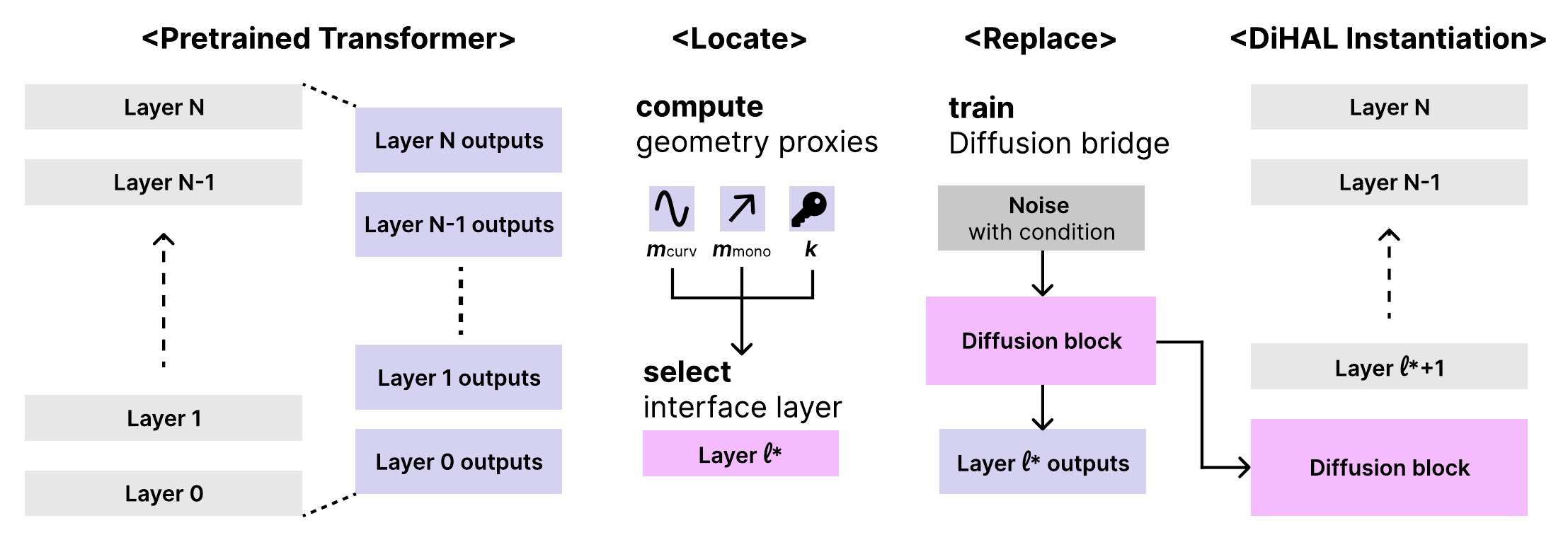}
\caption{
The Locate-and-Replace framework. DiHAL selects a hidden-state interface, replaces the preceding layers with a diffusion bridge, and retains the Transformer suffix and LM head.
}
\label{fig:overview}
\end{figure}

\section{Background}
\label{background}

Transformer language models take tokens as input, but most computation occurs in continuous hidden spaces. Given \(x_{1:T}\), an autoregressive model factorizes \(p(x_{1:T})=\prod_{t=1}^{T}p(x_t\mid x_{<t})\); each token is mapped to an embedding \(h_t^{(0)}=E(x_t)+p_t\), hidden states are updated as \(H^{(\ell)}=F_\ell(H^{(\ell-1)})\), and the final state is projected to vocabulary logits \(z_t=W_{\mathrm{LM}}h_t^{(L)}+b\) \citep{NIPS2017_3f5ee243}. Diffusion models generate samples by reversing a noising process, written in continuous space as \(dX_t=f(X_t,t)\,dt+g(t)\,dW_t\), where \(W_t\) is Brownian motion; the reverse process depends on the score \(\nabla_x\log p_t(x)\), which is approximated by a neural network \citep{song2021scorebased}. Applying diffusion to language thus requires choosing which representation to denoise. Discrete diffusion models corrupt tokens \citep{austin2021structured, hoogeboom2021argmax}, while continuous diffusion language models denoise token embeddings or learned latents \citep{li2022diffusionlm, strudel2023selfconditioned, lovelace2023latent}. Token-level diffusion suffers from recovery errors, as small denoising deviations can flip the recovered token \citep{zhang-etal-2025-diffusion, shen2026codarcontinuousdiffusionlanguage}, and learned latent diffusion still requires an interface for converting latents to text. We instead target internal transformer hidden states, where recovery becomes hidden-state reconstruction rather than direct token decoding.


\section{Method}
\label{sec:method}

We study where continuous diffusion should operate within a pretrained language model by treating hidden-state selection as an interface-design problem. Figure~\ref{fig:overview} illustrates our Locate-and-Replace procedure. We first motivate diffusion-friendly representations using geometric principles, then derive layer-wise proxies for selecting a suitable hidden-state interface, and finally test the selected interface by replacing the lower transformer layers with a conditional diffusion bridge while retaining the upper layers and LM head. We call this instantiation \texttt{DiHAL}. The bridge reconstructs an internal representation that the retained upper layers can directly process.

\subsection{Geometric Principles for Diffusion-Friendly Layer Selection}
\label{sec:theory}

We now make the notion of a \emph{diffusion-friendly} representation space more concrete. Intuitively, a good diffusion space should satisfy three properties: denoising should contract quickly toward the target representation distribution, remain stable under score-estimation error, and have low effective complexity, with variation concentrated in relatively few active directions.

We formalize the first two properties through overdamped Langevin dynamics, an idealized setting with clean convergence and stability guarantees. The third is captured by effective rank, which measures active variance directions. The theorem settings in this section are idealized: they motivate geometric surrogates rather than assume that transformer hidden states exactly satisfy them.

Throughout this section, \(W_2\) denotes the 2-Wasserstein distance and \(\mathcal P_2(\mathbb R^d)\) denotes probability measures with finite second moment. For a density \(p\), define \(U(x)=-\log p(x)\). We interpret strong convexity of \(U\) as a curvature-like restoring force toward high-density regions.

\begin{theorem}[Wasserstein contraction under strong log-concavity \citep{villani2009optimal,bakry2014analysis}]
\label{thm:logconcave}
Let \(\mu \in \mathcal{P}_2(\mathbb{R}^d)\) be a probability measure with density \(p\), and define \(U(x):=-\log p(x)\). Assume that \(U\in C^2(\mathbb{R}^d)\), \(U\) is \(m\)-strongly convex, i.e. \(\nabla^2 U(x)\succeq mI\) for all \(x\in\mathbb{R}^d\), for some \(m>0\), and that \(\nabla U\) is globally Lipschitz. Let \((X_t)_{t\ge0}\) satisfy the overdamped Langevin stochastic differential equation (SDE) \(dX_t=-\nabla U(X_t)\,dt+\sqrt{2}\,dW_t\), where \(W_t\) denotes Brownian motion.

Then \(\mu\) is an invariant distribution of \((X_t)\), and for every initial law \(\nu_0\in\mathcal P_2(\mathbb R^d)\),

$$
W_2(\nu_t,\mu)\le e^{-mt}W_2(\nu_0,\mu),
\qquad
\nu_t:=\mathcal L(X_t),
$$

where \(\mathcal L(X_t)\) denotes the distribution of \(X_t\).
The invariant distribution \(\mu\) is unique in \(\mathcal P_2(\mathbb R^d)\).
\end{theorem}

\begin{takeaway}
\textbf{Takeaway (Theorem~\ref{thm:logconcave}).}
Stronger curvature yields faster contraction under idealized Langevin dynamics, motivating curvature as a proxy for diffusion-friendly hidden-state distributions.
\end{takeaway}

Fast contraction alone is not enough. In practice, the score is unknown and must be estimated by a neural network. Theorem~\ref{thm:score} characterizes stability under such estimation error.

\begin{theorem}[Stability of an invariant measure under score perturbation] \label{thm:score} Let \(\mu\in\mathcal{P}_2(\mathbb{R}^d)\) have density \(p\) and define \(U(x):=-\log p(x)\). Assume that \(U\in C^2(\mathbb{R}^d)\) is \(m\)-strongly convex, and that \(\nabla U\) is globally Lipschitz. Let \(s(x):=\nabla\log p(x)=-\nabla U(x)\), and let \(\hat s:\mathbb{R}^d\to\mathbb{R}^d\) be globally Lipschitz and satisfy \(\sup_{x\in\mathbb{R}^d}\|\hat s(x)-s(x)\|\le \varepsilon\). Consider the two SDEs \(dX_t = s(X_t)\,dt + \sqrt{2}\,dW_t\) and \(d\hat X_t = \hat s(\hat X_t)\,dt + \sqrt{2}\,dW_t\). Assume that the second SDE admits an invariant distribution \(\hat\mu\). Then \(\hat\mu\in\mathcal{P}_2(\mathbb{R}^d)\) and \[ W_2(\hat\mu,\mu)\le \frac{\varepsilon}{m}. \] \end{theorem}

\begin{takeaway}
\textbf{Takeaway (Theorem~\ref{thm:score}).}
For a fixed score-estimation error, stronger curvature reduces the resulting distributional error, motivating curvature as a proxy for robustness to imperfect score estimation.
\end{takeaway}

Curvature, however, does not capture the full complexity of a representation: variation may still be distributed across many directions. We therefore complement curvature with effective rank as a proxy for dimensionality. For covariance matrix \(\Sigma\), define \(r_{\mathrm{eff}}(\Sigma):=\mathrm{tr}(\Sigma)/\|\Sigma\|\), where \(\mathrm{tr}(\Sigma)\) is the total variance and \(\|\Sigma\|\) is the largest covariance eigenvalue. Lower effective rank indicates that variation is concentrated along fewer active directions.

\begin{lemma}[Approximate manifold support implies low effective rank]
\label{lem:manifold_reff}
Let \(X\) be an \(\mathbb{R}^d\)-valued random variable with covariance \(\Sigma := \mathrm{Cov}(X)\). Assume there exist a \(k\)-dimensional \(C^2\) manifold \(\mathcal M \subset \mathbb R^d\) and a measurable map \(\Pi:\mathbb R^d \to \mathcal M\) such that \(\mathrm{tr}(\mathrm{Cov}(\Pi(X))) \le C_1 k\), \(\mathbb E\|X-\Pi(X)\|^2 \le C_2(\delta^2+\eta)\), and \(\|\Sigma\| \ge c > 0\).

Then \[ r_{\mathrm{eff}}(\Sigma) := \frac{\mathrm{tr}(\Sigma)}{\|\Sigma\|} \le \frac{2C_1}{c}\,k + \frac{2C_2}{c}\,(\delta^2+\eta). \] In particular, if \(\delta,\eta\) are controlled constants and \(\|\Sigma\|\) is bounded above and below by constants, then \[ r_{\mathrm{eff}}(\Sigma)=O(k). \] \end{lemma}

\begin{takeaway}
\textbf{Takeaway (Lemma~\ref{lem:manifold_reff}).}
Near a \(k\)-dimensional manifold, effective rank scales with intrinsic rather than ambient dimension, motivating it as a proxy for representation complexity.
\end{takeaway}

We next combine this dimension control with curvature-based concentration under strong log-concavity~\citep{bakry2014analysis,ledoux2001concentration}.

\begin{theorem}[Intrinsic dimension and effective representation complexity]
\label{thm:intrinsic}
Let \(\mu\in\mathcal{P}_2(\mathbb{R}^d)\) have density \(p(x)=Z^{-1}e^{-U(x)}\) for some \(U\in C^2(\mathbb{R}^d)\), and let \(\Sigma:=\mathrm{Cov}(\mu)\). Assume that \(\nabla^2 U(x)\succeq mI\) for all \(x\in\mathbb{R}^d\), for some \(m>0\), and that \(\nabla U\) is globally Lipschitz.

Assume moreover that there exist a \(k\)-dimensional \(C^2\) manifold \(\mathcal M\subset\mathbb R^d\) and a measurable map \(\Pi:\mathbb R^d\to \mathcal M\). For \(X\sim\mu\), suppose that \(\mathrm{tr}(\mathrm{Cov}(\Pi(X))) \le C_1 k\), \(\mathbb E\|X-\Pi(X)\|^2 \le C_2(\delta^2+\eta)\), and \(\|\Sigma\|\ge c_0>0\).

Then \[ r_{\mathrm{eff}}(\Sigma) \le \frac{2C_1}{c_0}\,k+\frac{2C_2}{c_0}(\delta^2+\eta), \] and hence \[ \bigl(\mathbb E\|X-\mathbb E X\|^2\bigr)^{1/2} = \sqrt{\mathrm{tr}(\Sigma)} \le \|\Sigma\|^{1/2} \left( \frac{2C_1}{c_0}\,k+\frac{2C_2}{c_0}(\delta^2+\eta) \right)^{1/2}. \]

Furthermore, there exists an absolute constant \(c>0\) such that for all \(t\ge0\), \[ \mathbb P\!\left( \left| \|X-\mathbb E X\|-\mathbb E\|X-\mathbb E X\| \right|\ge t \right) \le 2\exp(-cmt^2). \] \end{theorem}

\begin{takeaway}
\textbf{Takeaway (Theorem~\ref{thm:intrinsic}).}
Low intrinsic dimension controls effective representation complexity through effective rank, while stronger curvature yields tighter concentration around the typical representation scale. Together, they motivate effective rank and curvature as complementary geometric criteria for diffusion-friendly hidden states.
\end{takeaway}

Taken together, these results motivate three properties of diffusion-friendly representations: contraction, robustness to score-estimation error, and low effective complexity. While reverse diffusion uses time-dependent scores of noisy marginals rather than the fixed target score of Langevin dynamics, these principles provide geometric guidance for representation selection. We operationalize them using local covariance concentration, global precision-based stiffness, and effective rank, yielding testable geometric hypotheses for Transformer hidden states. We first validate the predicted relationships on controlled synthetic distributions, where each factor can be varied independently, and then examine whether the resulting proxies transfer to bridgeability across pretrained Transformer layers. Proofs of all results are provided in Appendix~\ref{proof}.

\subsection{Locate: Finding Diffusion-Friendly Layers}
\label{sec:locate}

We translate the three properties above into empirical spectral proxies computed from hidden-state activations (see Appendix~\ref{app:proxy-interpretation} for details). For each layer, let \(x \in \mathbb{R}^{M \times D}\) denote the activation matrix over \(M\) representation vectors \(D\) hidden dimensions.

We compute three statistics on \(x\). First, the local curvature proxy \(\hat m_{\mathrm{curv}}\) is obtained from the covariance of \(k\)-nearest-neighbor neighborhoods: \(m_{\mathrm{curv}}^{(i)}=1/\lambda_{\max}(\Sigma_{\mathrm{local}}^{(i)})\), with the layer-level value taken as the median; larger values indicate compact neighborhoods. Second, the monotonicity proxy \(\hat m_{\mathrm{mono}}\) captures global directional stiffness. With \(P=(\Sigma+\lambda I)^{-1}\) denoting the regularized precision of the empirical covariance \(\Sigma\), we compute \(m_{ij}=(x_i-x_j)^\top P(x_i-x_j)/\|x_i-x_j\|^2\) for sampled pairs and take the median. Third, effective intrinsic dimension is estimated as \(\hat k=r_{\mathrm{eff}}(\Sigma)=\mathrm{tr}(\Sigma)/\|\Sigma\|\). Diffusion-friendly layers should have large curvature-related proxies and small effective rank.

We combine these into a selection score: \(\mathrm{selection\_score}(\ell)=z(\log \hat m_{\mathrm{curv}}(\ell))+z(\log \hat m_{\mathrm{mono}}(\ell))-z(\log \hat k(\ell))\), where \(z(\cdot)\) denotes layer-wise z-score normalization. The score rewards curvature-related structure while penalizing effective rank. We define \emph{bridgeability} as reconstructability of a layer's hidden state by the diffusion bridge under a matched training protocol, measured by validation loss. The layer sweep evaluates how well the geometry score predicts bridgeability under this fixed protocol, and we select \(\ell^\ast=\arg\max_\ell \mathrm{selection\_score}(\ell)\). We use this score as a low-cost operational criterion for interface selection. Details are in Appendix~\ref{app:proxy-estimation}.

\begin{takeaway}
\textbf{Takeaway (Layer Selection).}
DiHAL favors layers with large local curvature \(\hat m_{\mathrm{curv}}\), large global stiffness \(\hat m_{\mathrm{mono}}\), and low effective rank \(\hat k\). Together, these proxies capture complementary aspects of diffusion-friendly representation geometry.
\end{takeaway}

\subsection{Replace: Hidden-State Diffusion Module}
\label{sec:replace}

Given the selected insertion layer \(\ell^\ast\), we test whether the selected interface can be reconstructed by conditional diffusion while remaining compatible with the pretrained suffix. Let \(F_{1:\ell^\ast}\) and \(F_{\ell^\ast+1:L}\) denote the original prefix and retained suffix, respectively. For input \(x\), the original model produces \(h_{\ell^\ast}=F_{1:\ell^\ast}(x)\), while the diffusion bridge reconstructs \(\hat h_{\ell^\ast}=D_\theta(c(x))\) from an embedding-derived condition \(c(x)\). The reconstructed state is passed through the retained suffix, \(h_L=F_{\ell^\ast+1:L}(\hat h_{\ell^\ast})\), and the LM head produces token probabilities.

At inference, \(D_\theta\) bypasses the lower layers and directly reconstructs \(\hat h_{\ell^\ast}\). We instantiate the bridge as a Stable-Diffusion-style U-Net: hidden states are projected to a latent tensor, conditionally denoised, and projected back to the model hidden space. Mirroring the training setup, the condition at step \(t\) contains only the observed prefix \(x_{\leq t}\), and the retained causal suffix produces the next-token distribution with consistent masking. A small-scale architecture ablation is provided in Appendix~\ref{app:bridge_architecture}.

We optimize hidden-state reconstruction using the diffusion loss \(\mathcal{L}_{\mathrm{diff}}=\mathbb{E}_{t,\epsilon}[\|\hat\epsilon_\theta(z_t,t,c)-\epsilon\|_2^2]\) and reconstruction loss \(\mathcal{L}_{\mathrm{rec}}=\|\hat h_{\ell^\ast}-h_{\ell^\ast}\|_2^2\), together with next-token and logit-distillation losses \(\mathcal{L}_{\mathrm{LM}}\) and \(\mathcal{L}_{\mathrm{KD}}\) to preserve suffix compatibility. The overall objective is \(\mathcal{L}=\mathcal{L}_{\mathrm{diff}}+\lambda_{\mathrm{rec}}\mathcal{L}_{\mathrm{rec}}+\lambda_{\mathrm{LM}}\mathcal{L}_{\mathrm{LM}}+\lambda_{\mathrm{KD}}\mathcal{L}_{\mathrm{KD}}\). Implementation details are in Appendix~\ref{app:bridge-details}.


\section{Experiments}
\label{sec:experiments}
 
Our experimental question is whether theory-motivated hidden-state geometry predicts fixed-budget diffusion bridgeability beyond layer depth, and whether this prediction transfers to fully trained models. We validate the geometric factors on controlled synthetic distributions, then test the geometry score across three pretrained backbones, including a held-out model family (Section~\ref{sec:transformer-bridgeability}). We compare diffusion with deterministic and alternative generative bridges, including a same-architecture control (Section~\ref{sec:bridge-controls}), and evaluate geometry-selected interfaces under full training (Section~\ref{sec:full-training}). We compare continuous diffusion interfaces under matched training budgets (Section~\ref{sec:continuous-diffusion-baselines}).

\subsection{Experimental Setup}
\label{sec:setup}

\paragraph{Backbones and activation corpus.}
We evaluate DiHAL on Llama-3.1-8B-Instruct~\citep{grattafiori2024llama3herdmodels} and Qwen3-8B~\citep{yang2025qwen3technicalreport}, with 32 and 36 Transformer layers, respectively, and hidden size 4096. We additionally evaluate OLMo-2-0425-1B~\citep{olmo20252olmo} (16 layers) as a held-out backbone from a distinct model family.
We extract embedding outputs and layer-wise hidden states from 300K \textit{Dolma v1.7} sequences~\citep{soldaini2024dolmaopencorpustrillion}.

\paragraph{Geometry estimation.}
For each candidate layer, we compute local compactness $\hat m_{\mathrm{curv}}$, global stiffness $\hat m_{\mathrm{mono}}$, and effective rank $\hat k=r_{\mathrm{eff}}(\Sigma)$.
We estimate them from 100 repeated 3K-example subsamples and assess score stability on 30 additional 500-example subsamples (Appendix~\ref{app:proxy-estimation}).
The selection-score coefficients are fixed rather than fitted to bridge losses.

\paragraph{Bridge training.}
We train one embedding-conditioned bridge per candidate layer for one epoch on 150K examples with a 9:1 train/validation split, using the same bridge family, data, optimization, and budget across layers.
Selected-layer bridges are trained on 300K examples for four epochs within 40 H100 GPU-hours.
Unless noted, we use Stable-Diffusion-v1.5-style components~\citep{Rombach_2022_CVPR}, with a frozen VAE and trainable U-Net and projections for hidden-state reconstruction; no image supervision or text-to-image conditioning is used (Appendix~\ref{app:bridge-details}).

\paragraph{Evaluation.}
We report Spearman and Kendall correlations between geometry score and negative bridge loss, along with predicted and observed best layers.
We test beyond depth using LOOCV prediction of log bridge loss.
For trained models, we report NLL, perplexity, and output-distribution KL on \textit{WikiText-103} and held-out \textit{Dolma v1.7}.
For diffusion LM comparisons, we report GPT-2 Gen.PPL and diversity (Appendix~\ref{app:generation-metrics}).
Details are in Appendix~\ref{app:exp-setup}.

\subsection{Geometry Predicts Bridgeability}
\label{sec:transformer-bridgeability}
 
\paragraph{Controlled synthetic validation.}
Section~\ref{sec:theory} motivates geometric factors for diffusion-friendly representations without assuming that Transformer hidden states satisfy the idealized theory.
We validate these factors on synthetic target distributions using the same fixed-budget Stable-Diffusion-style U-Net denoiser while varying one factor at a time.
Denoising difficulty follows the predicted direction, with Spearman correlations of $0.947$, $0.946$, and $0.968$ for global stiffness, effective rank, and local compactness, respectively, although the relationship weakens under strong multimodality.
We next test whether these proxies transfer to pretrained Transformer hidden spaces.

\paragraph{Layer-wise geometry.}
Figure~\ref{fig:layerwise_geometry} shows substantial geometric variation across depth for both backbones.
Input-adjacent layers generally exhibit stronger local compactness, while global stiffness and effective rank follow different trends.
The score selects layer~3 for Llama-3.1-8B and layer~2 for Qwen3-8B, while the lowest fixed-budget bridge loss occurs at layer~1 for both.
Thus, bridgeability is shallow, and we test whether geometry explains variation beyond this depth trend.
 
\begin{figure}[t]
    \centering
    \includegraphics[width=0.48\linewidth]{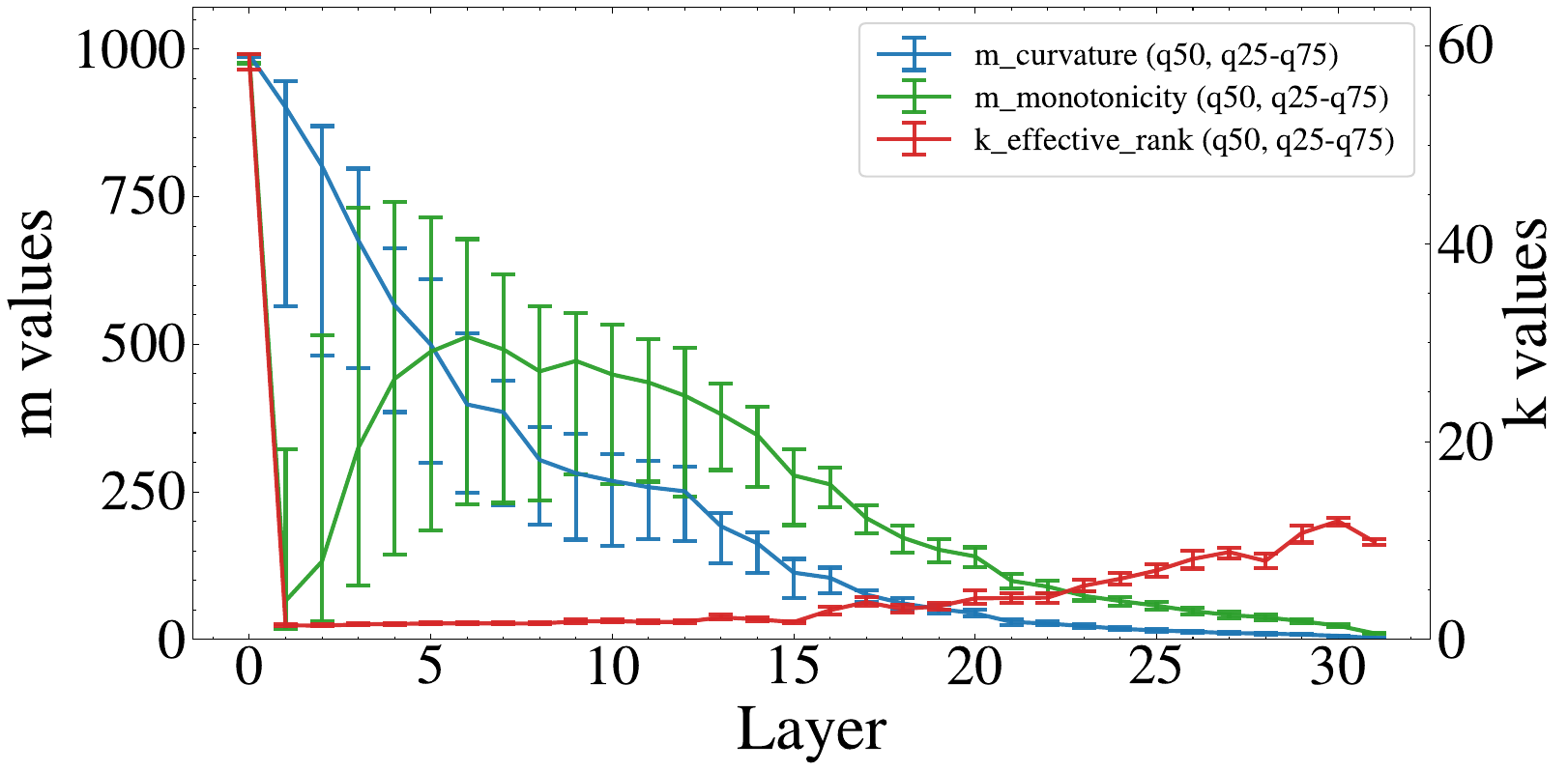}
    \hfill
    \includegraphics[width=0.48\linewidth]{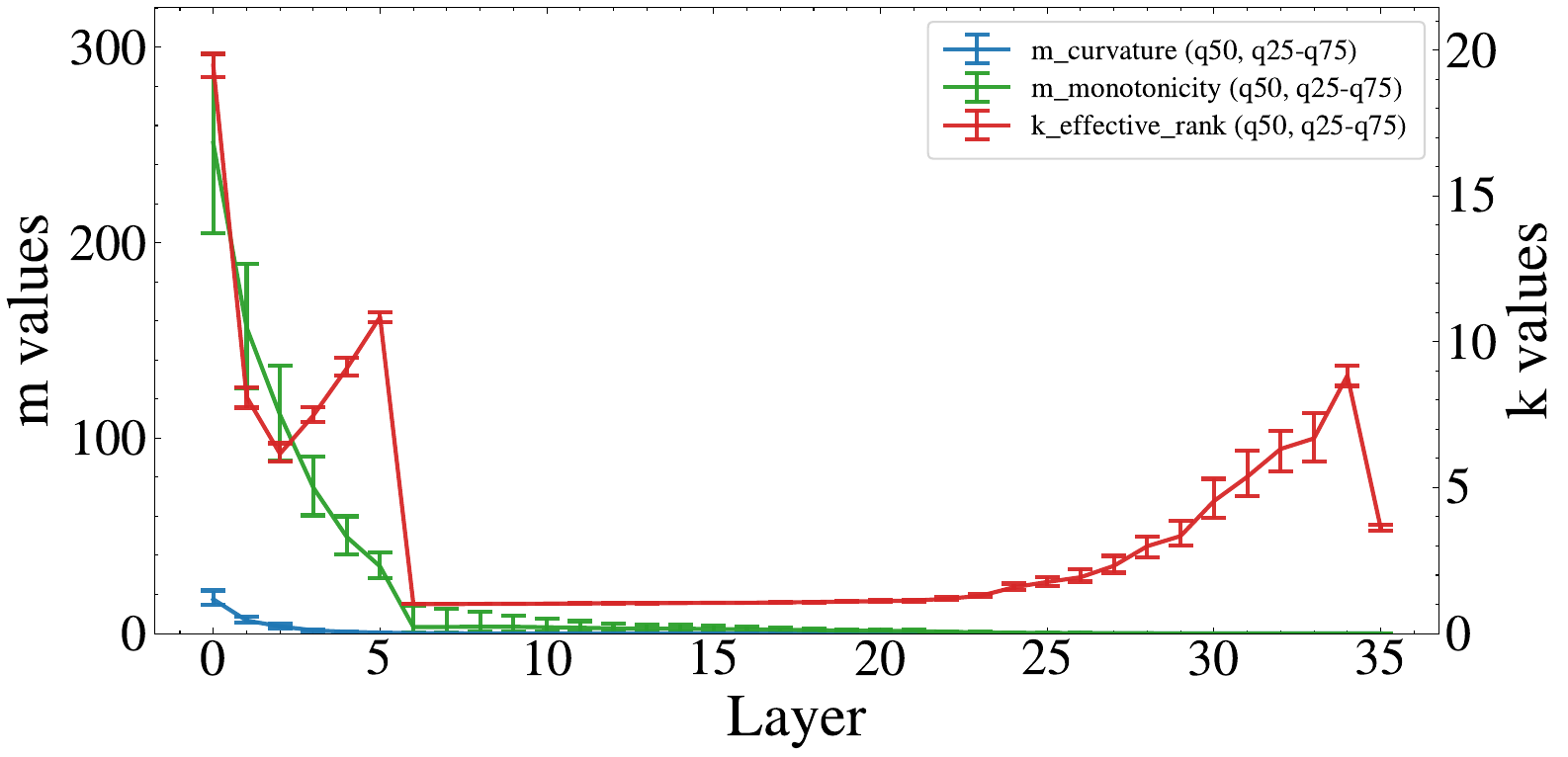}
    \caption{
    Layer-wise geometry for Llama-3.1-8B-Instruct (left) and Qwen3-8B (right).
The three geometric factors vary across depth.
The score selects layer~3 for Llama and layer~2 for Qwen.
    }
    \label{fig:layerwise_geometry}
\end{figure}
 
\paragraph{Fixed-budget layer sweep.}
Figure~\ref{fig:score_vs_val_loss} compares the geometry score with validation bridge loss across all candidate layers.
Higher-scoring layers generally exhibit lower bridge loss.
As shown in Table~\ref{tab:score_val_correlation}, the score achieves Spearman $\rho=0.9143\pm0.0069$ for Llama and $\rho=0.9267\pm0.0157$ for Qwen across repeated score estimates, with rank gaps of 2 and 1, respectively, between the geometry-selected and observed best layers. On OLMo, the geometry score achieves $\rho=0.965$ and correctly identifies the interior optimum at layer~6, where shallow-layer selection fails.
 
\begin{figure}[t]
    \centering
    \includegraphics[width=0.45\linewidth]{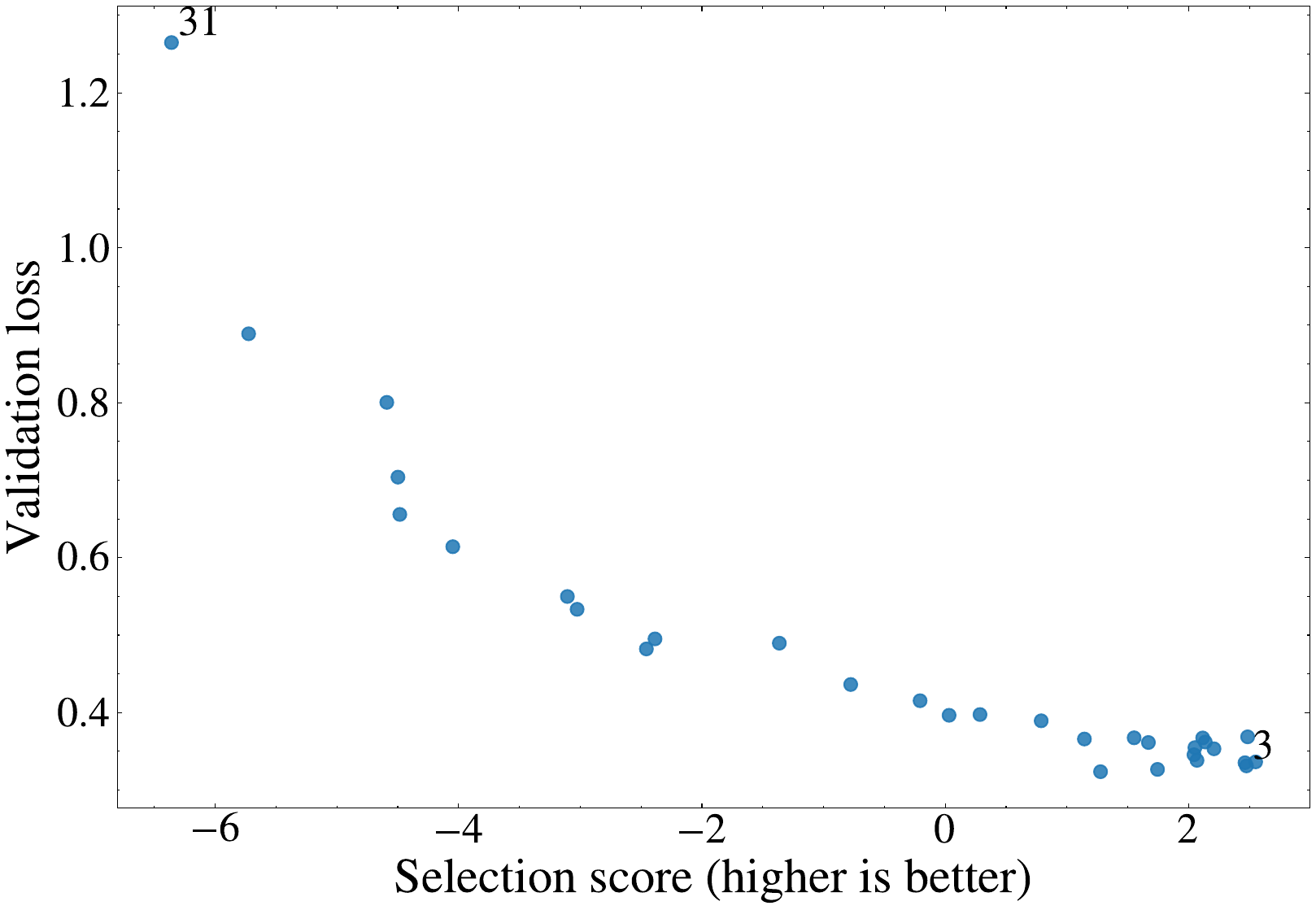}
    \hfill
    \includegraphics[width=0.45\linewidth]{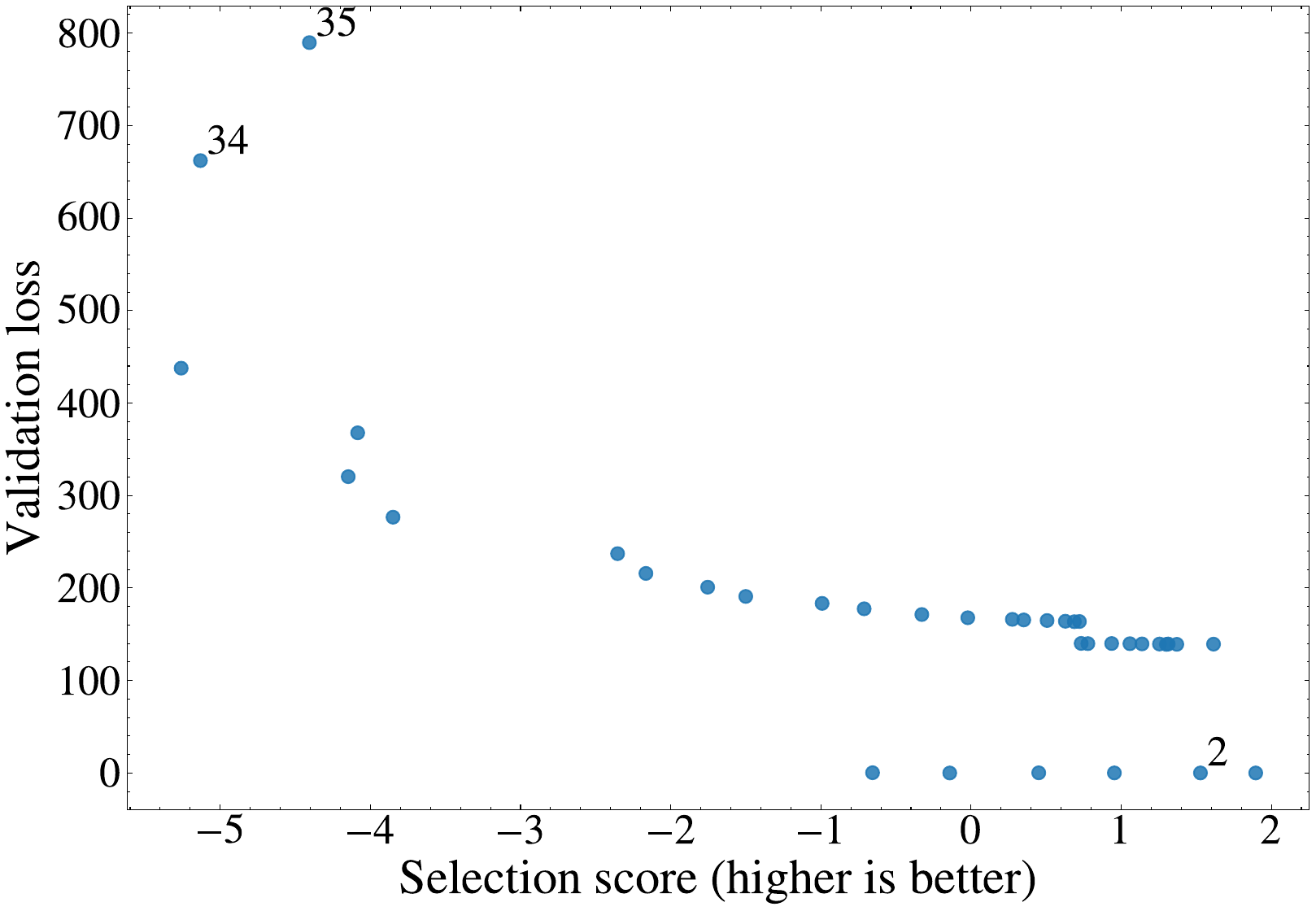}
    \caption{
    Geometry score versus bridge loss for Llama-3.1-8B-Instruct (left) and Qwen3-8B (right).
    Each point denotes a candidate layer.
    Higher scores generally correspond to lower fixed-budget loss.
    }
    \label{fig:score_vs_val_loss}
\end{figure}
 
\begin{table}[t]
\centering
\small
\caption{
Predictive value of geometry for fixed-budget bridgeability.
(a) Agreement between geometry score and bridgeability.
(b) Held-out prediction after controlling for normalized layer depth.
$\pm$ denotes standard deviation over 30 repeated score-estimation runs.
}
\label{tab:score_val_correlation}
\label{tab:depth_control}

\begin{minipage}[t]{0.48\linewidth}
\vspace{0pt}
\centering
\textbf{\scriptsize (a) Geometry--bridgeability agreement}

\vspace{3pt}

\resizebox{\linewidth}{!}{%
\begin{tabular}{lcccc}
\toprule
Backbone & Spearman $\rho\uparrow$ & Kendall $\tau\uparrow$ & Pred. & Obs. \\
\midrule
Llama-3.1-8B & $0.9143\pm0.0069$ & $0.8011\pm0.0086$ & 3 & 1 \\
Qwen3-8B     & $0.9267\pm0.0157$ & $0.8208\pm0.0256$ & 2 & 1 \\
OLMo-2-1B    & $0.9653\pm0.00042$ & $0.8830\pm0.0117$ & 6 & 6 \\
\bottomrule
\end{tabular}%
}
\end{minipage}
\hfill
\begin{minipage}[t]{0.50\linewidth}
\vspace{0pt}
\centering
\textbf{\scriptsize (b) Depth-controlled prediction}

\vspace{3pt}

\resizebox{\linewidth}{!}{%
\begin{tabular}{llrrr}
\toprule
Backbone & Predictor & RMSE $\downarrow$ & MAE $\downarrow$ & $R^2_{\mathrm{OOS}}\uparrow$ \\
\midrule
\multirow{2}{*}{Llama-3.1-8B}
& Quadratic depth  & 0.7537 & 0.3769 & 0.0915 \\
& + Geometry & \textbf{0.5400} & \textbf{0.1792} & \textbf{0.5336} \\
\midrule
\multirow{2}{*}{Qwen3-8B}
& Quadratic depth  & 1.6102 & 1.3248 & 0.6983 \\
& + Geometry & \textbf{1.5316} & \textbf{1.2550} & \textbf{0.7271} \\
\midrule
\multirow{2}{*}{OLMo-2-1B}
& Quadratic depth  & 1.4900 & 1.1870 & 0.7537 \\
& + Geometry & \textbf{1.3030} & \textbf{1.0450} & \textbf{0.8116} \\
\bottomrule
\end{tabular}%
}
\end{minipage}
\end{table}
 
\paragraph{Does geometry add information beyond depth?}
The preference for shallow layers raises an alternative explanation: the score may simply reflect that representations closer to the embedding condition are easier to reconstruct.
We therefore compare a quadratic depth model, \(y_\ell=\beta_0+\beta_1d_\ell+\beta_2d_\ell^2+\epsilon_\ell\), with the same model augmented by geometry, \(y_\ell=\beta_0+\beta_1d_\ell+\beta_2d_\ell^2+\beta_GG_\ell+\epsilon_\ell\), where \(y_\ell=\log\mathcal{L}_{\mathrm{bridge}}(\ell)\), using LOOCV (Appendix~\ref{app:additional-analyses}).


As shown in Table~\ref{tab:depth_control}, adding geometry reduces LOOCV RMSE by $28.4\%$ for Llama, $4.9\%$ for Qwen, and approximately $12.6\%$ for OLMo, while improving MAE and out-of-sample $R^2$ across all three backbones. Against CKA, hidden-state distance, and linear reconstruction error, geometry is strongest on Llama, while CKA is slightly better in RMSE and $R^2$ on Qwen (Appendix~\ref{app:depth-controls}); thus, geometry adds information beyond depth and recoverability without dominating every heuristic. 

\paragraph{Stability beyond one epoch.}
We continue training representative shallow, middle, and deep layers.
Normalized validation-loss ranks change by less than $10\%$ on average
for both 8B backbones, with geometry-selected layers remaining favored,
suggesting that bridgeability rankings are not solely driven by early optimization.

\begin{takeaway}
\textbf{Takeaway.}
The geometric proxies predict diffusion difficulty on synthetic data and transfer to transformer layers, providing signal beyond depth for locating diffusion-friendly interfaces.
\end{takeaway}

\subsection{Does Diffusion Help Beyond Generic Hidden-State Replacement?}
\label{sec:bridge-controls}
 
The preceding experiments establish where the bridge can reconstruct hidden states. We next ask whether diffusion provides an advantage over generic hidden-state replacement. We compare bridge families at the geometry-selected Llama layer~3 with sequence length $S=1024$. All bridges use the same 512 activation samples, 300 optimization steps, target layer, frozen suffix, LM head, and evaluation protocol. This reduced-budget experiment compares reconstruction and suffix compatibility across bridge families under matched conditions.
 
\begin{table}[t]
\centering
\small
\caption{
Comparison of hidden-state reconstruction methods under matched
fixed-budget protocols.
(a) Bridge-family comparison at Llama layer~3
(512 activation samples, 300 optimization steps).
(b) Same-architecture control on OLMo, reporting reconstruction
metrics at layer~6 and layer-wise wins across all 16 layers.
}
\label{tab:bridge_control}

\begin{minipage}[t]{0.59\linewidth}
\vspace{0pt}
\centering
\textbf{\scriptsize (a) Bridge-family comparison (Llama)}

\vspace{3pt}

\resizebox{\linewidth}{!}{%
\begin{tabular}{lrrrrr}
\toprule
Bridge & Params & MSE $\downarrow$ & KL $\downarrow$ & NLL $\downarrow$ & PPL $\downarrow$ \\
\midrule
DiHAL SD      & 946.3M & \textbf{0.0067} & \textbf{5.7446} & \textbf{7.6407} & \textbf{2,081.3} \\
Det. MLP      & 939.7M & 0.5665 & 9.0225  & 10.6760 & 43,302.7 \\
Cond. VAE     & 53.5M  & 0.4308 & 7.5690  & 9.6962  & 16,255.4 \\
Flow matching & 285.2M & 0.9399 & 10.3059 & 12.3001 & 219,713.1 \\
Hidden DDPM   & 272.6M & 1.5416 & 11.2132 & 13.1550 & 516,594.6 \\
\bottomrule
\end{tabular}%
}
\end{minipage}
\hfill
\begin{minipage}[t]{0.39\linewidth}
\vspace{0pt}
\centering
\textbf{\scriptsize (b) Same-architecture control (OLMo)}

\vspace{3pt}

\resizebox{\linewidth}{!}{%
\begin{tabular}{lrr}
\toprule
Metric & Diffusion & Deterministic \\
\midrule
Norm. MSE $\downarrow$       & \textbf{0.9042} & 0.9380 \\
Expl. variance $\uparrow$   & \textbf{9.58\%} & 6.20\% \\
Best layer                  & 6 & 6 \\
Layer wins $\uparrow$        & \textbf{13/16} & 3/16 \\
\bottomrule
\end{tabular}%
}
\end{minipage}
\end{table}
 
Table~\ref{tab:bridge_control}(a) shows that latent diffusion achieves the best reconstruction and suffix-level fidelity among the tested alternatives, reducing hidden-state MSE, teacher KL, and NLL by 98.8\%, 36.3\%, and 28.4\% relative to the parameter-matched deterministic MLP. Controlling for architecture on OLMo, diffusion achieves lower normalized MSE than direct prediction in 13 of 16 layers (Table~\ref{tab:bridge_control}(b); Appendix~\ref{app:same-architecture-control}).

\begin{takeaway}
\textbf{Takeaway.}
Latent diffusion improves reconstruction across bridge families and under same-architecture controls.
\end{takeaway}
 
\subsection{Does Geometry-Guided Selection Transfer to Full Training?}
\label{sec:full-training}
 
We fully train the bridge at the geometry-selected layer on the 300K-example corpus for four epochs within 40 H100 GPU-hours. We compare against a validation-loss oracle from the fixed-budget sweep and a worst-layer control. Unlike the oracle, which requires training every candidate bridge, the geometry score uses frozen activations alone.
 
\begin{table}[t]
\centering
\small
\setlength{\tabcolsep}{5pt}
\renewcommand{\arraystretch}{0.95}
\caption{
Fully trained DiHAL models, averaged over WikiText-103 and held-out Dolma ($N_{\mathrm{FE}}=20$).
The original model is a reference; the best DiHAL result per backbone is in bold.
}
\label{tab:full_training}
\begin{tabular}{llcrrr}
\toprule
Backbone & Model & Layer & NLL $\downarrow$ & PPL $\downarrow$ & KL $\downarrow$ \\
\midrule
\multirow{4}{*}{Llama-3.1-8B}
& Original (reference)   & -- & 1.87 & 6.47 & 0.00 \\
\cmidrule(l){2-6}
& Geometry-selected      & 3  & \textbf{4.91} & \textbf{135.64} & 0.73 \\
& Validation-loss oracle & 1  & 5.11 & 165.67 & \textbf{0.62} \\
& Worst layer            & 31 & 5.17 & 175.91 & 1.32 \\
\midrule
\multirow{4}{*}{Qwen3-8B}
& Original (reference)   & -- & 2.61 & 13.55 & 0.00 \\
\cmidrule(l){2-6}
& Geometry-selected      & 2  & 4.97 & 144.03 & \textbf{0.53} \\
& Validation-loss oracle & 1  & \textbf{4.94} & \textbf{139.77} & 0.54 \\
& Worst layer            & 35 & 5.23 & 186.79 & 1.46 \\
\bottomrule
\end{tabular}
\end{table}
 
Table~\ref{tab:full_training} shows that geometry-selected interfaces outperform worst-layer controls and remain close to the validation-loss oracles on the 8B backbones. On OLMo, geometry-selected layer~6 also outperforms layer~1 under full training, reducing PPL by $12.9\%$ and teacher KL by $22.2\%$ (Appendix~\ref{app:olmo-full-training}). These results demonstrate that geometry-guided interface selection transfers to full training. We further analyze suffix sensitivity (Appendix~\ref{app:suffix-sensitivity}) and inference costs in the appendix.

\begin{takeaway}
\textbf{Takeaway.}
Geometry-guided selection transfers to full training, remaining near the validation-loss oracle and outperforming the worst-layer control.
\end{takeaway}

\subsection{Does the Choice of Diffusion Interface Matter in Practice?}
\label{sec:continuous-diffusion-baselines}
 
We compare DiHAL-small on OLMo-2-1B with independently trained continuous diffusion language models targeting token embeddings (Diffusion-LM~\citep{li2022diffusionlm}; SED~\citep{strudel2023selfconditioned}), learned latents (LD4LG~\citep{lovelace2023latent}), and continuous-to-discrete recovery (CoDAR~\citep{shen2026codarcontinuousdiffusionlanguage}). All methods use 300K examples and 40 H100 GPU-hours for diffusion/recovery training. To reduce capacity differences, we use a randomly initialized U-Net on OLMo-2-1B, with an active bridge size comparable to CoDAR.

\begin{table}[t]
\centering
\small
\setlength{\tabcolsep}{4pt}
\renewcommand{\arraystretch}{0.95}
\caption{
Comparison of continuous diffusion interfaces under matched training budgets.
Parameter counts exclude retained pretrained components outside the bridge or recovery module.
}
\label{tab:continuous_diffusion_baselines}
\resizebox{0.85\linewidth}{!}{%
\begin{tabular}{llrrr}
\toprule
Method & Diffusion target & Params & Gen.PPL $\downarrow$ & Div. $\uparrow$ \\
\midrule
DiHAL-small (OLMo) & Hidden state $h_{\ell^\ast}$ & 339.5M & \textbf{138.32} & \textbf{0.5812} \\
Diffusion-LM & Token embeddings & 93.7M & 683.43 & 0.4324 \\
SED & Self-conditioned embeddings & 73.7M & 778.82 & 0.4942 \\
LD4LG & Learned text latent & 213.6M & 166.11 & 0.5797 \\
CoDAR & Continuous token/latent states & 360.9M & 144.83 & 0.4777 \\
\bottomrule
\end{tabular}%
}
\end{table}

Table~\ref{tab:continuous_diffusion_baselines} shows that DiHAL achieves the lowest Gen.PPL and highest diversity among the evaluated systems. Despite using a smaller, randomly initialized U-Net, DiHAL outperforms CoDAR on both metrics with comparable module capacity. The full DiHAL system retains 1.35B active parameters, so the comparison controls for training budgets and approximately matches module capacity with CoDAR, rather than total system capacity or inference FLOPs. Additional results on Qwen3-8B and capacity controls are provided in Appendix~\ref{app:capacity-controls}.

\section{Related Work}
\label{sec:related-work}

Diffusion language models adapt diffusion to text, where discreteness makes generation less direct than in images \citep{hoogeboom2022autoregressive}. Prior work includes discrete token diffusion \citep{austin2021structured,gong2023diffuseqv2}, masked diffusion with iterative refinement \citep{sahoo2024simple, shi2024simplified, nie2025large, ye2025dream}, and continuous diffusion over embeddings or learned latents \citep{li2022diffusionlm, lovelace2023latent}. Continuous methods avoid token corruption but require recovering tokens from denoised vectors, introducing projection or decoding errors \citep{wang2022language, li2022diffusionlm}. We study transformer hidden states as a diffusion-friendly denoising space.

Diffusion behavior depends on representation geometry, including curvature, conditioning, and intrinsic dimension \citep{ pidstrigach2022scorebasedgenerativemodelsdetect, Rombach_2022_CVPR}. In parallel, efficient generation has been studied through transformer compression, distillation, layer reduction, early exiting, and hybrid modules \citep{Fan2020Reducing, sanh2020distilbertdistilledversionbert, lenz2025jamba}. DiHAL connects these directions by identifying internal transformer representations suitable for diffusion-based replacement.


\section{Conclusion}
\label{sec:conclusion}

We studied continuous diffusion in language models as a hidden-state interface-selection problem, asking which internal representations are most compatible with diffusion rather than treating the denoising space as fixed. Across synthetic distributions and three pretrained Transformer backbones, theory-motivated geometric proxies strongly predict diffusion bridgeability, add predictive signal beyond layer depth, and identify a non-trivial interior optimum. Geometry-guided selection transfers to full training, while diffusion improves reconstruction over a same-architecture deterministic predictor under fixed-budget conditions. We instantiate this principle in \texttt{DiHAL}, which compares favorably with alternative continuous-diffusion interfaces under matched training budgets. These findings open a direction for future research beyond selecting existing representations: actively shaping representation geometry to construct diffusion-friendly spaces and developing language models that integrate diffusion and Transformer computation.

\bibliography{iclr2027_conference}
\bibliographystyle{iclr2027_conference}

\appendix
\section{Proofs of Theorems}
\label{proof}

\subsection*{Preliminaries}

\paragraph{Law of a random variable.}
Let $(\Omega,\mathcal{F},\mathbb{P})$ be a probability space and 
$X:\Omega \to \mathbb{R}^d$ be a random variable.
The law (distribution) of $X$, denoted by $\mathcal{L}(X)$,
is the probability measure defined by
\[
\mathcal{L}(X)(A) := \mathbb{P}(X \in A),
\quad A \in \mathcal{B}(\mathbb{R}^d).
\]

\paragraph{2-Wasserstein distance.}
Let $\mathcal{P}_2(\mathbb{R}^d)$ denote the set of probability measures 
with finite second moment. For $\nu,\mu \in \mathcal{P}_2(\mathbb{R}^d)$,
\[
W_2^2(\nu,\mu)
:=
\inf_{\pi \in \Pi(\nu,\mu)}
\int_{\mathbb{R}^d \times \mathbb{R}^d}
\|x-y\|^2 \, \pi(dx,dy),
\]
where $\Pi(\nu,\mu)$ denotes the set of couplings of $\nu$ and $\mu$.
Equivalently,
\[
W_2^2(\nu,\mu)
=
\inf
\left\{
\mathbb{E}\|X-Y\|^2
:\,
\mathcal{L}(X)=\nu,
\mathcal{L}(Y)=\mu
\right\}.
\]

In $\mathbb{R}^d$, an optimal coupling attaining the infimum exists.
If preferred, one may instead work with $\varepsilon$-optimal couplings
and let $\varepsilon \downarrow 0$ at the end.

\begin{lemma}[Strong convexity implies gradient monotonicity]
\label{lem:monotone}
Let $U:\mathbb{R}^d \to \mathbb{R}$ be differentiable and
$m$-strongly convex, i.e.
\[
U(y) \ge U(x)
+ \langle \nabla U(x), y-x \rangle
+ \frac{m}{2}\|y-x\|^2
\quad \text{for all } x,y.
\]
Then
\[
\langle \nabla U(x)-\nabla U(y), x-y \rangle
\ge m \|x-y\|^2
\quad \text{for all } x,y.
\]
\end{lemma}

\begin{proof}
Apply strong convexity twice, swapping $x$ and $y$:
\[
U(y) \ge U(x)
+ \langle \nabla U(x), y-x \rangle
+ \frac{m}{2}\|y-x\|^2,
\]
\[
U(x) \ge U(y)
+ \langle \nabla U(y), x-y \rangle
+ \frac{m}{2}\|x-y\|^2.
\]
Summing the two inequalities yields
\[
\langle \nabla U(x)-\nabla U(y), x-y \rangle
\ge m\|x-y\|^2.
\]
\end{proof}

\begin{lemma}[Gibbs invariance and uniqueness]
\label{lem:gibbs_invariant}
Assume that $U \in C^2(\mathbb{R}^d)$, $\nabla U$ is globally Lipschitz, and
\[
\nabla^2 U(x) \succeq mI
\qquad \text{for all } x \in \mathbb{R}^d
\]
for some $m>0$.
Let
\[
\mu(dx) = Z^{-1} e^{-U(x)}\,dx,
\qquad
Z := \int_{\mathbb{R}^d} e^{-U(x)}\,dx.
\]
Then $Z<\infty$, $\mu \in \mathcal P_2(\mathbb{R}^d)$, and $\mu$ is invariant for
\[
dX_t = -\nabla U(X_t)\,dt + \sqrt{2}\,dW_t.
\]
Moreover, if for every initial law $\nu_0 \in \mathcal P_2(\mathbb{R}^d)$ the law
$\nu_t := \mathcal L(X_t)$ satisfies
\[
W_2(\nu_t,\mu) \le e^{-mt} W_2(\nu_0,\mu),
\]
then $\mu$ is the unique invariant distribution in $\mathcal P_2(\mathbb{R}^d)$.
\end{lemma}

\begin{proof}
By strong convexity, for every $x\in\mathbb{R}^d$,
\[
U(x)
\ge
U(0) + \langle \nabla U(0), x\rangle + \frac m2 \|x\|^2.
\]
Completing the square, this implies
\[
e^{-U(x)} \le C \exp\!\left(-\frac m4 \|x\|^2\right)
\]
for some constant $C<\infty$. Hence $Z<\infty$ and $\mu$ has finite second moment.

Let
\[
L\varphi(x) := \Delta \varphi(x) - \langle \nabla U(x), \nabla \varphi(x)\rangle
\]
be the generator. For $\varphi \in C_c^\infty(\mathbb{R}^d)$, integration by parts gives
\[
\int_{\mathbb{R}^d} L\varphi(x)\, \mu(dx)
=
Z^{-1}
\int_{\mathbb{R}^d}
\bigl(
\Delta \varphi(x)
-
\langle \nabla U(x), \nabla \varphi(x)\rangle
\bigr)
e^{-U(x)}\,dx
=
0.
\]
Since this holds for all $\varphi \in C_c^\infty(\mathbb R^d)$, the stationary equation
$L^*\mu=0$ holds in the weak sense. Therefore $\mu$ is invariant.

Finally, if $\nu \in \mathcal P_2(\mathbb R^d)$ is any invariant distribution, then applying the
contraction estimate with $\nu_0=\nu$ yields
\[
W_2(\nu,\mu) = W_2(\nu_t,\mu) \le e^{-mt} W_2(\nu,\mu)
\qquad \text{for all } t\ge0.
\]
Fix any $t>0$. Since $e^{-mt}<1$, this implies $W_2(\nu,\mu)=0$, hence $\nu=\mu$.
\end{proof}

\subsection*{Proof of Theorem~\ref{thm:logconcave}}

By Lemma~\ref{lem:gibbs_invariant}, $\mu$ is an invariant distribution for
\[
dX_t = -\nabla U(X_t)\,dt + \sqrt{2}\,dW_t.
\]

Let $\nu_0 := \mathcal{L}(X_0)$.
Choose an optimal coupling $(X_0,Y_0)$ between $\nu_0$ and $\mu$ such that
\[
\mathbb{E}\|X_0-Y_0\|^2
=
W_2^2(\nu_0,\mu).
\]
(Alternatively, choose an $\varepsilon$-optimal coupling and let $\varepsilon \downarrow 0$.)
We take this initial coupling to be independent of the Brownian motion below.

Define the synchronous coupling:
\[
dX_t = -\nabla U(X_t)\,dt + \sqrt{2}\,dW_t,
\]
\[
dY_t = -\nabla U(Y_t)\,dt + \sqrt{2}\,dW_t,
\]
driven by the same Brownian motion $W_t$.

Since $Y_0 \sim \mu$ and $\mu$ is invariant, we have
\[
\mathcal{L}(Y_t)=\mu
\qquad \text{for all } t\ge0.
\]

Let $Z_t := X_t - Y_t$.
Subtracting the SDEs gives
\[
dZ_t = -(\nabla U(X_t)-\nabla U(Y_t))\,dt.
\]
Since the stochastic terms cancel, $Z_t$ has absolutely continuous paths and satisfies the
pathwise ODE
\[
\dot Z_t = -(\nabla U(X_t)-\nabla U(Y_t))
\qquad \text{for a.e. } t\ge0.
\]
Therefore the ordinary chain rule applies, and
\[
\frac{d}{dt}\|Z_t\|^2
=
2\langle Z_t,\dot Z_t\rangle
=
-2\langle Z_t, \nabla U(X_t)-\nabla U(Y_t) \rangle.
\]
By Lemma~\ref{lem:monotone},
\[
\frac{d}{dt}\|Z_t\|^2
\le -2m\|Z_t\|^2.
\]
By Gr\"onwall's inequality,
\[
\|Z_t\|^2
\le e^{-2mt}\|Z_0\|^2.
\]
Taking expectation yields
\[
\mathbb{E}\|X_t-Y_t\|^2
\le
e^{-2mt}
\mathbb{E}\|X_0-Y_0\|^2.
\]

Since $(X_t,Y_t)$ is a coupling of $\mathcal{L}(X_t)$ and $\mu$,
by definition of $W_2$,
\[
W_2^2(\mathcal{L}(X_t),\mu)
\le
\mathbb{E}\|X_t-Y_t\|^2.
\]
Therefore,
\[
W_2^2(\mathcal{L}(X_t),\mu)
\le
e^{-2mt}
W_2^2(\nu_0,\mu).
\]
Taking square roots gives
\[
W_2(\mathcal{L}(X_t),\mu)
\le
e^{-mt}
W_2(\nu_0,\mu).
\]

It remains to prove uniqueness of the invariant distribution in
$P_2(\mathbb{R}^d)$. Let $\pi\in P_2(\mathbb{R}^d)$ be any invariant
distribution of the same Langevin SDE. Applying the contraction estimate above
with $\nu_0=\pi$ gives, for every $t>0$,
\[
W_2(\pi,\mu)
=
W_2(\mathcal{L}(X_t),\mu)
\le
e^{-mt}W_2(\pi,\mu),
\]
where we used the invariance of $\pi$ to obtain $\mathcal{L}(X_t)=\pi$.
Since $m>0$ and $e^{-mt}<1$ for $t>0$, this implies
\[
W_2(\pi,\mu)=0.
\]
Hence $\pi=\mu$. Therefore $\mu$ is the unique invariant distribution in
$P_2(\mathbb{R}^d)$.
\qed

\subsection*{Proof of Theorem~\ref{thm:score}}

We prove stability of an invariant measure under a uniformly bounded perturbation of the score function.

Let $\mu \in \mathcal{P}_2(\mathbb{R}^d)$ have density $p$ and define
\[
U(x) := -\log p(x).
\]
Assume that $U \in C^2(\mathbb{R}^d)$ is $m$-strongly convex, i.e.
\[
\nabla^2 U(x) \succeq m I
\quad \text{for all } x \in \mathbb{R}^d,
\]
for some $m>0$, and that $\nabla U$ is globally Lipschitz.

Define the true score
\[
s(x) := \nabla \log p(x) = -\nabla U(x).
\]
Let $\hat{s}$ be a globally Lipschitz function satisfying
\[
\sup_{x \in \mathbb{R}^d} \|\hat{s}(x) - s(x)\| \le \varepsilon.
\]

Consider the SDEs
\[
dX_t = s(X_t)\,dt + \sqrt{2}\,dW_t,
\qquad
d\hat X_t = \hat{s}(\hat X_t)\,dt + \sqrt{2}\,dW_t,
\]
and assume that the second SDE admits an invariant distribution $\hat\mu$.
By Lemma~\ref{lem:gibbs_invariant}, $\mu$ is an invariant distribution of the first SDE.

We first show that $\hat\mu \in \mathcal P_2(\mathbb R^d)$.
Since $s=-\nabla U$ and $U$ is $m$-strongly convex, Lemma~\ref{lem:monotone} applied with
$y=0$ gives
\[
\langle x, s(x)-s(0)\rangle \le -m\|x\|^2
\qquad \text{for all } x\in\mathbb R^d.
\]
Hence
\[
\langle x,s(x)\rangle
=
\langle x, s(x)-s(0)\rangle + \langle x,s(0)\rangle
\le
-m\|x\|^2 + \|s(0)\|\,\|x\|.
\]
Using $\|\hat s(x)-s(x)\|\le \varepsilon$, we obtain
\[
\langle x,\hat s(x)\rangle
\le
\langle x,s(x)\rangle + \varepsilon \|x\|
\le
-m\|x\|^2 + (\|s(0)\|+\varepsilon)\|x\|.
\]
By Young's inequality, there exists a constant $C_0<\infty$ such that
\[
\langle x,\hat s(x)\rangle
\le
-\frac m2 \|x\|^2 + C_0
\qquad \text{for all } x\in\mathbb R^d.
\]

Let
\[
\hat L f(x) := \langle \hat s(x), \nabla f(x)\rangle + \Delta f(x)
\]
be the generator of the perturbed SDE, and set
\[
V(x):=\|x\|^2.
\]
Then
\[
\hat L V(x)
=
2\langle x,\hat s(x)\rangle + 2d
\le
-m\|x\|^2 + C_1
=
-mV(x)+C_1
\]
for some constant $C_1<\infty$.

We now justify the use of this unbounded Lyapunov function by truncation.
Let $\psi_R\in C^2([0,\infty))$ be nondecreasing and concave, with
\[
0\le \psi_R'\le 1,\qquad \psi_R''\le 0,
\]
such that
\[
\psi_R(r)=r \quad \text{for } r\le R,
\qquad
\psi_R(r)=\text{constant} \quad \text{for } r\ge 2R.
\]
Define
\[
V_R(x):=\psi_R(V(x)).
\]
Since $V_R \in C_b^2(\mathbb{R}^d)$ and the perturbed SDE has globally Lipschitz drift, the associated Markov semigroup $(\widehat P_t)_{t\ge 0}$ has generator $\hat L$. Hence
\[
\left.\frac{d}{dt}\widehat P_t V_R\right|_{t=0}
=
\hat L V_R .
\]
By invariance of $\hat\mu$,
\[
\int_{\mathbb{R}^d} \widehat P_t V_R(x)\,\hat\mu(dx)
=
\int_{\mathbb{R}^d} V_R(x)\,\hat\mu(dx)
\quad\text{for all }t\ge 0.
\]
Differentiating at $t=0$ gives
\[
\int_{\mathbb{R}^d} \hat L V_R(x)\,\hat\mu(dx)
=
0.
\]
By the chain rule for $\hat L$,
\[
\hat L V_R
=
\psi_R'(V)\hat L V
+
\psi_R''(V)\|\nabla V\|^2.
\]
Since $\psi_R''\le 0$, we have
\[
\hat L V_R
\le
\psi_R'(V)\hat L V
\le
\psi_R'(V)(-mV+C_1).
\]
Therefore
\[
0
=
\int \hat L V_R\,d\hat\mu
\le
-m\int \psi_R'(V)V\,d\hat\mu
+
C_1\int \psi_R'(V)\,d\hat\mu.
\]
Since $0\le \psi_R'\le 1$, this implies
\[
m\int \psi_R'(V)V\,d\hat\mu
\le
C_1.
\]
Moreover, $\psi_R'(V)=1$ whenever $V\le R$, and hence
\[
m\int_{\{V\le R\}} V\,d\hat\mu
\le
C_1.
\]
Letting $R\to\infty$ and using monotone convergence gives
\[
\int \|x\|^2\,\hat\mu(dx)
=
\int V\,d\hat\mu
\le
\frac{C_1}{m}
<\infty.
\]
Thus $\hat\mu\in\mathcal P_2(\mathbb R^d)$.

Now choose an optimal coupling $(X_0,\hat X_0)$ of $\mu$ and $\hat\mu$ such that
\[
\mathbb E\|X_0-\hat X_0\|^2 = W_2^2(\mu,\hat\mu).
\]
We take this initial coupling to be independent of the Brownian motion below.
Drive both SDEs by the same Brownian motion $(W_t)_{t\ge0}$ and define
\[
Z_t := \hat X_t - X_t.
\]
Subtracting the two SDEs yields
\[
dZ_t = \bigl(\hat s(\hat X_t) - s(X_t)\bigr)\,dt.
\]
Since the stochastic terms cancel, $Z_t$ has absolutely continuous paths and satisfies the
pathwise ODE
\[
\dot Z_t = \hat s(\hat X_t) - s(X_t)
\qquad \text{for a.e. } t\ge0.
\]
Therefore the ordinary chain rule applies, and
\[
\frac{d}{dt}\|Z_t\|^2
=
2\langle Z_t, \hat s(\hat X_t)-s(X_t)\rangle.
\]

We decompose
\[
\hat s(\hat X_t)-s(X_t)
=
\bigl(\hat s(\hat X_t)-s(\hat X_t)\bigr)
+
\bigl(s(\hat X_t)-s(X_t)\bigr).
\]
Hence
\[
\frac{d}{dt}\|Z_t\|^2
=
2\langle Z_t,\hat s(\hat X_t)-s(\hat X_t)\rangle
+
2\langle Z_t,s(\hat X_t)-s(X_t)\rangle.
\]

For the first term, using the uniform bound on the score error,
\[
2\langle Z_t,\hat s(\hat X_t)-s(\hat X_t)\rangle
\le
2\varepsilon \|Z_t\|.
\]
For the second term, Lemma~\ref{lem:monotone} implies
\[
\langle \hat X_t-X_t,\nabla U(\hat X_t)-\nabla U(X_t)\rangle
\ge
m\|\hat X_t-X_t\|^2.
\]
Since $s=-\nabla U$, we get
\[
\langle Z_t,s(\hat X_t)-s(X_t)\rangle
\le
-m\|Z_t\|^2.
\]
Combining the two bounds gives
\[
\frac{d}{dt}\|Z_t\|^2
\le
-2m\|Z_t\|^2 + 2\varepsilon \|Z_t\|.
\]
By Young's inequality,
\[
2\varepsilon \|Z_t\|
\le
m\|Z_t\|^2 + \frac{\varepsilon^2}{m}.
\]
Substituting this into the differential inequality, we obtain
\[
\frac{d}{dt}\|Z_t\|^2
\le
-m\|Z_t\|^2 + \frac{\varepsilon^2}{m}.
\]

Multiplying both sides by $e^{mt}$ and integrating from $0$ to $t$ yields
\[
\|Z_t\|^2
\le
e^{-mt}\|Z_0\|^2
+
\frac{\varepsilon^2}{m^2}\bigl(1-e^{-mt}\bigr)
\qquad \text{a.s.}
\]
Taking expectations gives
\[
\mathbb E\|Z_t\|^2
\le
e^{-mt}\mathbb E\|Z_0\|^2
+
\frac{\varepsilon^2}{m^2}\bigl(1-e^{-mt}\bigr).
\]

Since $X_0\sim\mu$ and $\hat X_0\sim\hat\mu$ are invariant initial laws, we have
\[
X_t\sim\mu,
\qquad
\hat X_t\sim\hat\mu
\qquad \text{for all } t\ge0.
\]
Therefore $(X_t,\hat X_t)$ is a coupling of $\mu$ and $\hat\mu$, so
\[
W_2^2(\hat\mu,\mu)
\le
\mathbb E\|\hat X_t-X_t\|^2
=
\mathbb E\|Z_t\|^2.
\]
Combining this with the previous estimate gives
\[
W_2^2(\hat\mu,\mu)
\le
e^{-mt}W_2^2(\hat\mu,\mu)
+
\frac{\varepsilon^2}{m^2}\bigl(1-e^{-mt}\bigr).
\]
Letting $t\to\infty$, we obtain
\[
W_2^2(\hat\mu,\mu)\le \frac{\varepsilon^2}{m^2}.
\]
Taking square roots gives
\[
W_2(\hat\mu,\mu)\le \frac{\varepsilon}{m}.
\]

\qed

\subsection*{Proof of Lemma~\ref{lem:manifold_reff}}
\label{app:manifold_reff_proof}

Let
\[
\Sigma := \mathrm{Cov}(X).
\]
Define
\[
R := X-\Pi(X).
\]
Since
\[
\mathbb E X = \mathbb E \Pi(X) + \mathbb E R,
\]
we may write
\[
X-\mathbb E X
=
\bigl(\Pi(X)-\mathbb E\Pi(X)\bigr)
+
\bigl(R-\mathbb E R\bigr).
\]
Therefore, using the inequality $\|u+v\|^2 \le 2\|u\|^2 + 2\|v\|^2$, we obtain
\[
\mathrm{tr}(\Sigma)
=
\mathbb E\|X-\mathbb E X\|^2
\le
2\,\mathbb E\|\Pi(X)-\mathbb E\Pi(X)\|^2
+
2\,\mathbb E\|R-\mathbb E R\|^2.
\]

For the first term, by assumption,
\[
\mathbb E\|\Pi(X)-\mathbb E\Pi(X)\|^2
=
\mathrm{tr}(\mathrm{Cov}(\Pi(X)))
\le
C_1 k.
\]

For the second term, since covariance is dominated by the second moment,
\[
\mathbb E\|R-\mathbb E R\|^2
=
\mathrm{tr}(\mathrm{Cov}(R))
\le
\mathbb E\|R\|^2.
\]
Using the approximation assumption,
\[
\mathbb E\|R\|^2
=
\mathbb E\|X-\Pi(X)\|^2
\le
C_2(\delta^2+\eta).
\]

Combining the two bounds yields
\[
\mathrm{tr}(\Sigma)
\le
2C_1 k + 2C_2(\delta^2+\eta).
\]

By definition,
\[
r_{\mathrm{eff}}(\Sigma)
=
\frac{\mathrm{tr}(\Sigma)}{\|\Sigma\|}.
\]
Using the non-degeneracy assumption $\|\Sigma\|\ge c>0$, we obtain
\[
r_{\mathrm{eff}}(\Sigma)
\le
\frac{2C_1 k + 2C_2(\delta^2+\eta)}{c}
=
\frac{2C_1}{c}\,k + \frac{2C_2}{c}\,(\delta^2+\eta).
\]

In particular, if $\delta,\eta=O(1)$ and $\|\Sigma\|$ is bounded below by a positive
constant, then
\[
r_{\mathrm{eff}}(\Sigma)=O(k+1).
\]
If additionally $k\ge 1$, this simplifies to
\[
r_{\mathrm{eff}}(\Sigma)=O(k).
\]

\qed

\subsection*{Proof of Theorem~\ref{thm:intrinsic}}

Let $X\sim \mu$, and denote
\[
\bar x := \mathbb E X,
\qquad
\Sigma := \mathrm{Cov}(X)
=
\mathbb E\big[(X-\bar x)(X-\bar x)^\top\big].
\]

By the trace identity,
\[
\mathbb E\|X-\bar x\|^2
=
\mathbb E\,\mathrm{tr}\!\big((X-\bar x)(X-\bar x)^\top\big)
=
\mathrm{tr}\!\Big(\mathbb E (X-\bar x)(X-\bar x)^\top\Big)
=
\mathrm{tr}(\Sigma).
\]

Next, by the assumptions of Lemma~\ref{lem:manifold_reff}, we have
\[
\mathrm{tr}(\mathrm{Cov}(\Pi(X))) \le C_1 k,
\qquad
\mathbb E\|X-\Pi(X)\|^2 \le C_2(\delta^2+\eta),
\qquad
\|\Sigma\| \ge c_0 > 0.
\]
Therefore Lemma~\ref{lem:manifold_reff} yields
\[
r_{\mathrm{eff}}(\Sigma)
=
\frac{\mathrm{tr}(\Sigma)}{\|\Sigma\|}
\le
\frac{2C_1}{c_0}\,k + \frac{2C_2}{c_0}\,(\delta^2+\eta).
\]
Since
\[
\mathrm{tr}(\Sigma)=\|\Sigma\|\,r_{\mathrm{eff}}(\Sigma),
\]
it follows that
\[
\mathrm{tr}(\Sigma)
\le
\|\Sigma\|
\left(
\frac{2C_1}{c_0}\,k + \frac{2C_2}{c_0}\,(\delta^2+\eta)
\right).
\]
Using $\mathbb E\|X-\bar x\|^2=\mathrm{tr}(\Sigma)$, we obtain
\[
\bigl(\mathbb E\|X-\bar x\|^2\bigr)^{1/2}
=
\sqrt{\mathrm{tr}(\Sigma)}
\le
\|\Sigma\|^{1/2}
\left(
\frac{2C_1}{c_0}\,k + \frac{2C_2}{c_0}\,(\delta^2+\eta)
\right)^{1/2}.
\]

It remains to prove the concentration statement.
Since
\[
p(x)=Z^{-1}e^{-U(x)}
\]
and
\[
\nabla^2 U(x)\succeq mI
\qquad \text{for all } x\in\mathbb R^d,
\]
the Bakry--\'Emery criterion implies that $\mu$ satisfies a logarithmic Sobolev
inequality with constant of order $1/m$, up to the standard normalization convention.
Consequently, by the Herbst argument, every $1$-Lipschitz function $f:\mathbb R^d\to\mathbb R$
satisfies a Gaussian concentration bound of the form
\[
\mathbb P\!\left(
|f(X)-\mathbb E f(X)|\ge t
\right)
\le
2\exp(-cmt^2)
\qquad \text{for all } t\ge0,
\]
where $c>0$ is an absolute constant.

Now define
\[
f(x):=\|x-\bar x\|.
\]
For any $x,y\in\mathbb R^d$,
\[
|f(x)-f(y)|
=
\bigl|\|x-\bar x\|-\|y-\bar x\|\bigr|
\le
\|x-y\|,
\]
so $f$ is $1$-Lipschitz. Applying the preceding concentration bound to this $f$ gives
\[
\mathbb P\!\left(
\left|
\|X-\bar x\|-\mathbb E\|X-\bar x\|
\right|
\ge t
\right)
\le
2\exp(-cmt^2)
\qquad \text{for all } t\ge0.
\]

This proves the claimed concentration inequality and completes the proof.
\qed

\begin{corollary}

By Jensen's inequality,
\[
\mathbb E\|X-\mathbb E X\|
\le
\sqrt{\mathbb E\|X-\mathbb E X\|^2}
=
\sqrt{\mathrm{tr}(\Sigma)}.
\]
Hence
\[
\mathbb P\!\left(
\|X-\mathbb E X\|
\ge
\sqrt{\mathrm{tr}(\Sigma)}+t
\right)
\le
2\exp(-cmt^2).
\]
In particular, if $\delta,\eta$ are bounded by constants and $\|\Sigma\|\asymp 1$, then there exists
a constant $C>0$ such that
\[
\sqrt{\mathrm{tr}(\Sigma)}\le C\sqrt{k},
\]
and therefore
\[
\mathbb P\!\left(
\|X-\mathbb E X\|\ge C\sqrt{k}+t
\right)
\le
2\exp(-cmt^2).
\]
\end{corollary}


\section{Interpretation of the Geometric Proxies}
\label{app:proxy-interpretation}

In this appendix, we explain why the empirical quantities used in our layer-selection procedure may be viewed as proxies for the theoretical geometric terms introduced in Section~\ref{sec:theory}.
Our goal is not to recover the true strong-convexity constant or the true dimension of the manifold of the distribution of layerwise representations.
We instead seek observables that have the same \emph{functional roles} in diffusion-friendly geometry.

\paragraph{Theoretical role of $m$.}
In our theory, the quantity $m$ appears as the strong-convexity constant of the potential $U(x)=-\log p(x)$, namely
\[
\nabla^2 U(x) \succeq m I.
\]
Which means that in every direction the potential has at least curvature $m$.
A larger $m$ implies stronger contraction of the matching Langevin dynamics and greater robustness to score perturbations.
For this reason, in the diffusion perspective $m$ measures how strongly the representation distribution exhibits restoring geometry toward high-density regions.

In practice, however, the exact density $p(x)$ of layer activations is unknown, and only a finite sample of activation vectors is available.
As a result, neither $U(x)$ nor its Hessian $\nabla^2 U(x)$ can be computed exactly.
This motivates the use of empirical curvature-related proxies.

\paragraph{Why $\hat m_{\mathrm{mono}}$ is an $m$-like quantity.}
Our monotonicity proxy is defined from the empirical covariance $\Sigma$ and its precision matrix $P=\Sigma^{-1}$.
For sampled pairs $(x_i,x_j)$, we compute
\[
m_{ij}
=
\frac{(x_i-x_j)^\top P (x_i-x_j)}{\|x_i-x_j\|^2},
\]
and summarize these values by a robust statistic such as the median.

This quantity is closely related to the role of $m$ in the Gaussian case.
Indeed, if the layerwise representation distribution were exactly Gaussian with mean $\mu$ and precision matrix $P$, then
\[
U(x)
=
\frac12 (x-\mu)^\top P (x-\mu) + \mathrm{const},
\]
and hence
\[
\nabla^2 U(x)=P.
\]
In that setting, the strong-convexity constant is precisely
\[
m=\lambda_{\min}(P).
\]
Moreover, for any displacement vector $\delta$, the Rayleigh quotient
\[
\frac{\delta^\top P\delta}{\|\delta\|^2}
\]
measures directional curvature under the quadratic potential.
Therefore, $\hat m_{\mathrm{mono}}$ can be interpreted as an empirical summary of the typical directional stiffness of the representation space.
Although it is not an exact lower bound on the Hessian, it captures the same global geometric intuition: layers with larger $\hat m_{\mathrm{mono}}$ behave as if they are embedded in a more strongly restoring global geometry.

\paragraph{Why $\hat m_{\mathrm{curv}}$ is also an $m$-like quantity.}
Our local curvature proxy is based on the covariance of local neighborhoods.
For each anchor point, we compute the covariance matrix $\Sigma_{\mathrm{local}}$ of its $k$-nearest neighbors and define a local score proportional to
\[
\frac{1}{\lambda_{\max}(\Sigma_{\mathrm{local}})}.
\]
The layer-level statistic $\hat m_{\mathrm{curv}}$ is then obtained by aggregating these local values across anchors.

This proxy is motivated by a local quadratic approximation.
If a small neighborhood of the representation distribution is approximately Gaussian, then its local density may be written as
\[
p_{\mathrm{local}}(x)
\propto
\exp\!\left(
-\frac12 (x-\mu_{\mathrm{local}})^\top P_{\mathrm{local}} (x-\mu_{\mathrm{local}})
\right),
\]
with $P_{\mathrm{local}} \approx \Sigma_{\mathrm{local}}^{-1}$.
Under this approximation, the local Hessian of the negative log-density is given by $P_{\mathrm{local}}$, and the corresponding local strong-convexity scale is
\[
\lambda_{\min}(P_{\mathrm{local}})
=
\lambda_{\min}(\Sigma_{\mathrm{local}}^{-1})
=
\frac{1}{\lambda_{\max}(\Sigma_{\mathrm{local}})}.
\]
Thus, $\hat m_{\mathrm{curv}}$ measures how compact and sharply curved the local representation neighborhoods are.
Larger values indicate smaller local spread along the most variable direction, which is consistent with the intuition of stronger local restoring geometry.

\paragraph{Why $\hat k$ reflects intrinsic dimension.}
The theoretical quantity $k$ is intended to capture the intrinsic complexity of the representation distribution, rather than its ambient dimension $d$.
In our empirical procedure, we measure this using the effective rank of the covariance:
\[
\hat k
=
r_{\mathrm{eff}}(\Sigma)
=
\frac{\mathrm{tr}(\Sigma)}{\|\Sigma\|}.
\]

This quantity has a direct interpretation in idealized low-dimensional settings.
Suppose the data are supported exactly on a $k$-dimensional isotropic subspace with covariance eigenvalues
\[
\lambda_1=\cdots=\lambda_k=\sigma^2,
\qquad
\lambda_{k+1}=\cdots=\lambda_d=0.
\]
Then
\[
r_{\mathrm{eff}}(\Sigma)
=
\frac{k\sigma^2}{\sigma^2}
=
k.
\]
Hence, in this ideal case, the effective rank exactly recovers the intrinsic dimension.

More generally, when the representation distribution is concentrated near a low-dimensional manifold or subspace, the covariance spectrum typically contains a small number of dominant eigenvalues and many small residual ones.
In such cases, $r_{\mathrm{eff}}(\Sigma)$ measures the effective number of active variation directions.
For our purposes, this is the relevant notion of intrinsic dimension, since diffusion complexity depends not on the nominal ambient dimension but on how many directions carry meaningful variation.

\paragraph{Interpretation and limitation.}
Taken together, $\hat m_{\mathrm{mono}}$ and $\hat m_{\mathrm{curv}}$ serve as global and local proxies for curvature-related restoring geometry, while $\hat k$ serves as an operational measure of intrinsic dimension.
These quantities are not exact estimators of the theoretical constants in Section~\ref{sec:theory}.
Instead, they are empirical surrogates designed to preserve the same geometric design principle:
diffusion-friendly layers should be locally compact, globally stable, and effectively low-dimensional.

This distinction is important.
Our empirical layer-selection score should therefore be interpreted as an operational criterion motivated by theory, rather than as a direct numerical estimate of the exact constants appearing in the theorems.


\section{Geometric Proxy Estimation and Layer Selection}
\label{app:proxy-estimation}
 
This appendix describes how the empirical geometric proxies used in the \emph{Locate} stage are computed and combined into a layer-selection score.
Appendix~\ref{app:proxy-interpretation} explains why these quantities can be interpreted as theory-motivated surrogates for curvature and intrinsic dimension; here we focus on their computation.
 
\subsection{Representation Extraction}
\label{app:representation-extraction}
 
For each Transformer layer, we begin from hidden activations
\[
\mathrm{acts}\in\mathbb{R}^{N\times S\times H},
\]
where $N$ is the number of sequences, $S$ the sequence length, and $H$ the hidden dimension (e.g., $[32,1024,4096]$).
These activations are converted into a set of representation vectors
\[
x\in\mathbb{R}^{M\times H}
\]
using one of the following extraction modes.
 
\paragraph{Mean pooling.}
Each sequence is mapped to one vector by masked averaging over valid tokens,
\[
x_n=\frac{\sum_{t=1}^{S}a_{nt}\,h_{nt}}{\sum_{t=1}^{S}a_{nt}},
\]
where $a_{nt}\in\{0,1\}$ is the attention mask and $h_{nt}\in\mathbb{R}^H$ is the hidden state at position $t$.
 
\paragraph{Last-token pooling.}
Each sequence is represented by the hidden state of the last token.
 
\paragraph{Tokenwise extraction.}
All valid token representations are flattened across sequences and positions, and a random subset is retained if needed to control memory and computation.
 
\paragraph{Mask handling and default setting.}
Padding positions (\texttt{attention\_mask}$=0$) are excluded from all computations, and special tokens are kept whenever the attention mask marks them as valid.
For reproducibility, we saved the number of sequences, sequence length, hidden dimension, and number and dimension of vectors before and after projection for every layer.
The main experiments use \texttt{mean} pooling and retain at most \texttt{max\_seqs}$=3{,}000$ sequences and \texttt{max\_tokens}$=200{,}000$ valid tokens per layer-wise geometry run.
 
\subsection{Random Projection}
\label{app:preprocessing-random-projection}
 
Because hidden representations are high-dimensional, the main geometry pipeline applies a fixed random projection before estimating the proxies.
We construct a Gaussian matrix $R\in\mathbb{R}^{H\times D}$, orthonormalize its columns with a QR decomposition, and use the same matrix for every layer within a backbone.
The projected representations are
\[
z=xR\in\mathbb{R}^{M\times D},
\]
with $D=64$ in the main experiments.
This projection is a computational sketch that reduces the cost and improves the numerical stability of repeated covariance-based estimation; it does not assert that the original hidden states are intrinsically $64$-dimensional.
All proxies in the main layer-selection experiments are computed from these projected representations, and Appendix~\ref{app:projection-robustness} shows that the resulting layer ranking is stable for larger projection dimensions and in the full hidden space.
 
\subsection{Local Curvature Proxy}
\label{app:local-curvature-proxy}
 
To capture local compactness, we sample anchor points from the projected representations and, for each anchor $z_i$, identify its $k$ nearest neighbors $\mathcal{N}_i$ under Euclidean distance.
We compute the ridge-regularized local covariance
\[
\Sigma_{\mathrm{local}}^{(i)}=\mathrm{Cov}(\mathcal{N}_i)+\lambda I
\]
and define
\[
m_{\mathrm{curv}}^{(i)}=\frac{1}{\lambda_{\max}\!\left(\Sigma_{\mathrm{local}}^{(i)}\right)}.
\]
This quantity is large when a neighborhood is compact even along its most variable direction.
The layer-level statistic is the median over anchors,
\[
\hat m_{\mathrm{curv}}=\mathrm{median}_i\bigl(m_{\mathrm{curv}}^{(i)}\bigr),
\]
and we additionally record the 25th and 75th percentiles for the uncertainty bands in Figure~\ref{fig:layerwise_geometry}.
 
\subsection{Monotonicity Proxy}
\label{app:monotonicity-proxy}
 
To capture global directional stiffness, let
\[
\Sigma=\mathrm{Cov}(z)+\lambda I,
\qquad
P=\Sigma^{-1}
\]
be the ridge-regularized covariance and its precision matrix.
For sampled pairs $(z_i,z_j)$, we compute the Rayleigh-quotient-type quantity
\[
m_{ij}=\frac{(z_i-z_j)^\top P(z_i-z_j)}{\|z_i-z_j\|^2},
\]
which measures how strongly a pairwise displacement is penalized by the global precision geometry.
The layer-level statistic is
\[
\hat m_{\mathrm{mono}}=\mathrm{median}_{(i,j)}(m_{ij}),
\]
for which we also report quantiles and 95\% bootstrap percentile intervals.
 
\subsection{Effective Rank}
\label{app:effective-rank}
 
We estimate effective intrinsic dimension with the effective rank of the projected covariance,
\[
\hat k=r_{\mathrm{eff}}(\Sigma)=\frac{\mathrm{tr}(\Sigma)}{\|\Sigma\|},
\]
where $\|\Sigma\|$ is the largest eigenvalue of $\Sigma$.
Smaller values indicate that variation is concentrated along fewer dominant directions.
Vectors are centered but not individually $\ell_2$-normalized, the covariance is regularized by $10^{-3}I$, and covariance construction and eigendecomposition are performed in float64, with small negative eigenvalues from numerical error clamped to zero.
Effective-rank values near one therefore indicate that the projected covariance is dominated by its leading direction; they do not imply that the original hidden-state space is one-dimensional.
We additionally record the participation ratio and the 25th, 50th, and 75th percentiles of $\hat k$ as supplementary diagnostics.

 
\section{Diffusion Bridge}
\label{app:bridge}
 
\subsection{Architecture and Conditioning}
\label{app:bridge-details}
 
The conditioning signal $c(x)$ is the embedding output $\mathrm{embedding\_out}\in\mathbb{R}^{B\times S\times H}$, i.e., the hidden state immediately before the first Transformer block of the source model.
A learned condition projection maps it to the 768-dimensional \texttt{encoder\_hidden\_states} interface of the pretrained U-Net.
No CLIP or other pretrained text encoder is used; the U-Net conditioning interface is repurposed to receive the embedding-derived representation.
 
The bridge follows a Stable-Diffusion-style latent-denoising design.
The target hidden state $h_{\ell^\ast}$ is mapped to an image-like tensor (Appendix~\ref{app:hidden-image-interface}), encoded by a frozen VAE into a latent $z_{\ell^\ast}$, noised to $z_t$ at timestep $t$ with a fixed scheduler, and denoised by the U-Net, which predicts $\hat\epsilon_\theta(z_t,t,c)$.
The denoised latent is decoded and projected back to the hidden space to obtain $\hat h_{\ell^\ast}$.
 
During training, we update the U-Net together with the bridge-specific projections (hidden-to-image, image-to-hidden, and condition projections), while the VAE and diffusion scheduler remain fixed.
The Replace stage is thus trained as conditional hidden-state reconstruction rather than text-to-image generation.
For the alignment losses in Section~\ref{sec:replace}, the language-modeling term is
$\mathcal{L}_{\mathrm{LM}}=\mathrm{CE}\bigl(p_{\mathrm{bridge}}(x_{i+1}\mid x_{\le i}),x_{i+1}\bigr)$,
and the temperature-scaled distillation term is
$\mathcal{L}_{\mathrm{KD}}=T^2\,\mathrm{KL}\bigl(\mathrm{softmax}(o_{\mathrm{teacher}}/T)\,\|\,\mathrm{softmax}(o_{\mathrm{bridge}}/T)\bigr)$,
where $o_{\mathrm{teacher}}$ and $o_{\mathrm{bridge}}$ are the teacher and bridge logits.
 
\subsection{Hidden-to-Image Interface}
\label{app:hidden-image-interface}
 
For a boundary activation $H^{(\ell^\ast)}\in\mathbb{R}^{B\times S\times H}$, where $H=4096$ for Llama and Qwen and $H=2048$ for OLMo, a learned $1\times1$ projection maps the hidden dimension to the three input channels of the SD-v1.5 VAE, and the sequence positions are rearranged bijectively in row-major order into a grid whose size depends on the sequence length: $32\times32$ for $S=1024$ and $16\times32$ for $S=512$.
No positions are removed or doubled.
After latent denoising, the inverse rearrangement restores the original token order, and a learned projection maps the three channels back to the hidden dimension.
 
The grid is only an architectural adapter for the pretrained latent-diffusion backbone; it does not assume that language representations have a natural two-dimensional topology or RGB semantics.
Likewise, the three-channel projection is imposed by the VAE interface and is not a statement about the intrinsic dimensionality of language representations.
Positions are processed jointly because the targets are contextual sequence representations, and processing each token independently would impose a factorized bridge.
Our evaluation criterion is therefore functional rather than coordinate-wise: the reconstructed state should remain usable by the frozen suffix and LM head.
 
\subsection{Autoregressive Inference}
\label{app:inference-procedure}
\label{app:autoregressive-generation}
 
The VAE and U-Net are spatially non-causal within their inputs, but no future token information is provided to the bridge.
At generation step $t$, the bridge condition contains the embeddings of the observed prefix $x_{\le t}$; the remaining positions of the fixed-length grid are masked and carry no token information.
Starting from diffusion noise $\xi_t$ that is independent of the text, the bridge produces a boundary representation for the observed prefix,
\[
\widehat H_t^{(\ell^\ast)}=D_\theta(C_t,\xi_t),
\qquad
C_t=[E(x_1),\ldots,E(x_t)].
\]
The reconstructed prefix is passed through the unchanged suffix with its standard causal attention mask, and only the final observed position is used to predict the next token,
\[
p_\theta(x_{t+1}\mid x_{\le t},\xi_t)
=
\operatorname{softmax}\Bigl(
W_{\mathrm{LM}}\,
F_{>\ell^\ast}^{\mathrm{causal}}\bigl(\widehat H_t^{(\ell^\ast)}\bigr)_t
\Bigr).
\]
The predicted token is appended, and the bridge is rerun on the extended prefix with fresh noise.
The bridge can therefore mix information among observed positions but cannot access unobserved future text, so DiHAL remains autoregressive at the token-generation interface.
Rerunning the bridge at every step is also a major source of the inference overhead in Appendix~\ref{sec:inference-cost}.
 
\subsection{Backbone Ablation}
\label{app:bridge_architecture}
 
We use a Stable-Diffusion-style U-Net because the replacement module must solve a high-dimensional conditional denoising problem over continuous representations rather than a next-token prediction task.
Table~\ref{tab:hidden_bridge_valid_loss_layer16_with_teacher} compares denoising backbones under the same validation setting at layer~16 ($S=512$, $N=500$), a layer chosen independently of the geometry score.
The SD-style U-Net achieves the lowest validation loss: $1163.59$, a $17.6\%$ reduction relative to the MLP-based hidden DDPM ($1411.36$), and reductions of $58.4\%$ and $68.3\%$ relative to the Transformer-based and 1D-convolutional DDPM bridges.
Deterministic and other generative bridges are compared at the selected layer in Section~\ref{sec:bridge-controls}.
We do not claim that this architecture is generally optimal but it is the best-performing backbone within this small-scale ablation.

\begin{table}[t]
\centering
\small
\caption{Validation bridge loss of denoising backbones at layer~16 ($S=512$, $N=500$).}
\label{tab:hidden_bridge_valid_loss_layer16_with_teacher}
\begin{tabular}{lr}
\toprule
Bridge backbone & Mean validation loss $\downarrow$ \\
\midrule
Stable-Diffusion-style U-Net       & \textbf{1163.59} \\
Hidden DDPM (MLP)                  & 1411.36 \\
Hidden DDPM (Transformer)          & 2799.28 \\
Hidden DDPM (1D convolution)       & 3672.30 \\
\bottomrule
\end{tabular}
\end{table}
 

\subsection{Same-Architecture Deterministic Control}
\label{app:same-architecture-control}

To distinguish the contribution of diffusion from that of the bridge architecture, we compare the diffusion bridge with a deterministic direct-prediction control across all 16 decoder layers of OLMo-2-0425-1B~\citep{olmo20252olmo}. Both use the same U-Net, projections, and embedding-derived conditioning under a matched fixed-budget protocol. The deterministic control directly predicts the target hidden state without diffusion noising. We evaluate reconstruction using normalized MSE and examine its association with the unchanged geometry score.

\begin{table}[t]
\centering
\small
\caption{Normalized MSE at representative OLMo layers (left) and geometry--bridgeability correlations over the full and shallow layer ranges (right). Correlations use negative normalized MSE.}
\label{tab:olmo_same_architecture}
\begin{minipage}[t]{0.42\linewidth}
\centering
\begin{tabular}{crr}
\toprule
Layer & Diffusion & Direct \\
\midrule
0 & \textbf{0.9705} & 0.9735 \\
1 & 0.9476 & \textbf{0.9420} \\
2 & 0.9704 & \textbf{0.9635} \\
6 & \textbf{0.9042} & 0.9380 \\
15 & 0.9977 & \textbf{0.9970} \\
\bottomrule
\end{tabular}
\end{minipage}\hfill
\begin{minipage}[t]{0.54\linewidth}
\centering
\resizebox{\linewidth}{!}{%
\begin{tabular}{lrr}
\toprule
Metric & Diffusion & Direct \\
\midrule
Spearman $\rho$ (L0--L15) & \textbf{0.965} & 0.862 \\
Spearman $\rho$ (L0--L7) & \textbf{0.905} & 0.814 \\
Kendall $\tau$ (L0--L15) & \textbf{0.883} & 0.745 \\
\bottomrule
\end{tabular}%
}
\end{minipage}
\end{table}

Diffusion achieves lower normalized MSE at layers~0 and~3--14, while direct prediction performs slightly better at layers~1, 2, and~15. Both methods attain their minimum error at layer~6. Table~\ref{tab:olmo_same_architecture} also shows that geometry predicts reconstruction difficulty for both methods, with stronger correlations for diffusion. This control evaluates fixed-budget reconstruction rather than full-budget downstream fidelity with the frozen suffix.

\section{Experimental Setup Details}
\label{app:exp-setup}
 
\paragraph{Data sampling and activation dumps.}
We construct the activation corpus from a local raw-text copy of \textit{Dolma v1.7}.
The pipeline samples up to 300{,}000 sequences with stratified round-robin sampling over 700 source files, with up to 500 samples per file and random seed 42.
Each sequence is tokenized with the source-model tokenizer and passed through the backbone to store the embedding output, all decoder-layer outputs, input IDs, and attention masks as float16 shards.
Special tokens are retained under the tokenizer's default handling, padding is tracked by the attention mask, and special tokens are removed only when token IDs are decoded back into text.
 
\paragraph{Geometry estimation.}
Geometry is estimated from mask-aware mean-pooled representations projected to $D=64$ with a fixed random projection shared across layers (Appendix~\ref{app:proxy-estimation}).
The default pipeline uses ridge coefficient $10^{-3}$, $k=64$ nearest neighbors and 512 anchor points for local curvature, 200{,}000 sampled pairs for monotonicity, and bootstrap resampling with 95\% confidence intervals.
Proxies are estimated from repeated subsamples as described in Section~\ref{sec:setup}.
 
\paragraph{Fixed-budget layer sweep.}
For the two 8B backbones, bridge training uses a deterministic shard-level split with validation ratio $0.1$ and split seed 42.
Each candidate layer uses a 150{,}000-example subset, split into up to 135{,}000 training and 15{,}000 validation examples, and is trained for one epoch with batch size 4, learning rate $3\times10^{-5}$, AdamW, and FP16 mixed precision, for at most 33{,}750 optimization steps.
For Llama and Qwen, bridgeability is measured by validation bridge loss; training losses are recorded only as optimization diagnostics. For the OLMo sweep, we report normalized MSE as the reconstruction measure.
 
\paragraph{Full training.}
For Llama and Qwen, the final bridge at each selected layer is trained on the full 300K-example corpus for four epochs within 40 H100 GPU-hours, which is also the diffusion/recovery training budget given to each baseline in Section~\ref{sec:continuous-diffusion-baselines}.
The auxiliary losses $\mathcal{L}_{\mathrm{LM}}$ and $\mathcal{L}_{\mathrm{KD}}$ are applied only in the final epoch.

\paragraph{Held-out backbone.}
OLMo-2-0425-1B~\citep{olmo20252olmo} has 16 decoder layers and hidden size 2048. We apply the geometry scoring rule without retuning and conduct a fixed-budget diffusion and deterministic reconstruction sweep over layers~0--15. We further compare full-budget diffusion bridges at layers~1 and~6 with the pretrained suffix and LM head frozen. The smaller randomly initialized U-Net variant used for generation is described in Appendix~\ref{app:capacity-controls}.

\paragraph{Compute resources.}
All bridge training, including the layer sweep and full training, is conducted on NVIDIA H100 GPUs.
Because activation dumps store float16 hidden states for every layer, storage grows approximately linearly with the number of layers and sampled sequences.
 
 
\section{Additional Layer-Selection Analyses}
\label{app:additional-analyses}
 
\subsection{Depth and Embedding-Proximity Controls}
\label{app:depth-controls}
 
For the two 8B backbones, the full layer sweep exhibits a strong shallow-to-deep trend, which raises the possibility that the geometry score mainly reflects proximity to the embedding representation.
For each layer $\ell$, we predict $y_\ell=\log\mathcal{L}_{\mathrm{bridge}}(\ell)$ with a base model in normalized depth $d_\ell$,
\[
y_\ell=\beta_0+\beta_1 d_\ell+\beta_2 d_\ell^2+\epsilon_\ell,
\]
and augment it with one auxiliary predictor $B_\ell$,
\[
y_\ell=\beta_0+\beta_1 d_\ell+\beta_2 d_\ell^2+\beta_B B_\ell+\epsilon_\ell,
\]
where $B_\ell$ is embedding-to-layer linear CKA, hidden-state distance from the embedding output, held-out linear reconstruction error from the embedding output, or the geometry score.
All results use leave-one-layer-out cross-validation.
 
\begin{table}[t]
\centering
\small
\caption{Depth-controlled prediction of log bridge loss. Each auxiliary predictor is added to the same quadratic normalized-depth model and evaluated with leave-one-layer-out cross-validation. The best result for each backbone is in bold.}
\label{tab:full_depth_controls}
\begin{tabular}{llrrr}
\toprule
Backbone & Predictor & RMSE $\downarrow$ & MAE $\downarrow$ & $R^2_{\mathrm{OOS}}\uparrow$ \\
\midrule
\multirow{5}{*}{Llama-3.1-8B}
& Quadratic depth               & 0.7537 & 0.3769 & 0.0915 \\
& Depth + CKA                   & 0.6566 & 0.3011 & 0.3106 \\
& Depth + hidden distance       & 0.8903 & 0.4311 & $-0.2676$ \\
& Depth + linear reconstruction & 0.7626 & 0.3739 & 0.0700 \\
& Depth + geometry              & \textbf{0.5400} & \textbf{0.1792} & \textbf{0.5336} \\
\midrule
\multirow{5}{*}{Qwen3-8B}
& Quadratic depth               & 1.6102 & 1.3248 & 0.6983 \\
& Depth + CKA                   & \textbf{1.5024} & 1.2684 & \textbf{0.7373} \\
& Depth + hidden distance       & 1.7306 & 1.4733 & 0.6515 \\
& Depth + linear reconstruction & 1.6539 & 1.3689 & 0.6817 \\
& Depth + geometry              & 1.5316 & \textbf{1.2550} & 0.7271 \\
\bottomrule
\end{tabular}
\end{table}
 
For Llama-3.1-8B, geometry gives the best held-out prediction on all three metrics, reducing RMSE by $28.4\%$ relative to quadratic depth alone.
For Qwen3-8B, geometry performs better only with depth and provides the lowest MAE at the same time, while CKA is slightly better in RMSE and out-of-sample $R^2$.
Geometry therefore provides predictive information beyond layer index and simple linear recoverability, but does not uniformly dominate every embedding-proximity heuristic.
 
We also repeat the analysis on the first eight layers only, where the range of depth is much narrower.
Adding geometry reduces RMSE by $14.1\%$ for Llama but increases it by $3.9\%$ for Qwen.
This restricted analysis supports a conservative interpretation: geometry is informative beyond the global depth trend, particularly for Llama, but shallow depth remains an important predictor of bridgeability.
 
\subsection{Projection-Dimension Robustness}
\label{app:projection-robustness}
 
To test whether the $D=64$ sketch (Appendix~\ref{app:preprocessing-random-projection}) determines the layer ranking, we recompute the complete geometry score with larger projection dimensions and in the full hidden space.
Table~\ref{tab:projection_robustness} reports the Spearman correlation of each ranking with the default $D=64$ ranking.
Although the numerical values of individual covariance statistics change with the projection dimension, the layer ranking used for selection is highly stable.
 
\begin{table}[t]
\centering
\small
\caption{Robustness of the geometry-score layer ranking to the projection dimension, measured against the default $D=64$ ranking.}
\label{tab:projection_robustness}
\begin{tabular}{cc}
\toprule
Projection dimension & Spearman $\rho$ \\
\midrule
128               & 0.9894 \\
256               & 0.9876 \\
512               & 0.9866 \\
4096 (full space) & 0.9874 \\
\bottomrule
\end{tabular}
\end{table}
 

\subsection{Generalization to a Held-Out Backbone}
\label{app:olmo-layer-selection}

We apply the fixed geometry scoring rule to OLMo-2-0425-1B without tuning its coefficients. Unlike the two 8B backbones, OLMo exhibits non-monotonic bridgeability: the diffusion reconstruction error reaches its minimum at layer~6 rather than at the shallowest layers. The geometry score selects the same layer, with a score of $0.934$ compared with $-1.193$ at layer~0 and $-0.815$ at layer~1. The geometry score achieves Spearman $\rho=0.965$ and Kendall $\tau=0.883$ over all 16 layers, and Spearman $\rho=0.905$ within layers~0--7.

To account for the non-monotonic depth trend, we evaluate both linear and quadratic depth models using leave-one-layer-out cross-validation, with normalized MSE as the prediction target. Table~\ref{tab:olmo_depth_controls} reports the results across all layers, after excluding the final layer, and within the shallow regime. Geometry improves prediction in all four settings, including after controlling for quadratic depth.

\begin{table}[t]
\centering
\small
\caption{Depth-controlled prediction on OLMo. RMSE is computed on normalized MSE and scaled by 100 for readability. Reduction is relative to the depth-only model.}
\label{tab:olmo_depth_controls}
\begin{tabular}{llrrr}
\toprule
Layers & Depth model & Depth only & + Geometry & Reduction $\uparrow$ \\
\midrule
L0--L15 & Linear    & 3.018 & \textbf{1.152} & 61.8\% \\
L0--L15 & Quadratic & 1.490 & \textbf{1.303} & 12.6\% \\
L0--L14 & Quadratic & 1.529 & \textbf{0.955} & 37.5\% \\
L0--L7  & Quadratic & 1.715 & \textbf{1.093} & 36.3\% \\
\bottomrule
\end{tabular}
\end{table}

\subsection{Full-Training Results on OLMo}
\label{app:olmo-full-training}

We further examine whether the interior optimum transfers to full-budget training by comparing the geometry-selected layer~6 against layer~1, a shallow-layer reference, with the retained Transformer suffix and LM head frozen. Table~\ref{tab:olmo_full_training} shows that geometry-guided selection improves all three downstream metrics, reducing PPL by $12.9\%$ and teacher KL by $22.2\%$.

\begin{table}[t]
\centering
\small
\caption{Full-training comparison on OLMo-2-1B. Both bridges use the retained frozen suffix and LM head corresponding to their insertion layer.}
\label{tab:olmo_full_training}
\begin{tabular}{lcrrr}
\toprule
Model & Layer & NLL $\downarrow$ & PPL $\downarrow$ & KL $\downarrow$ \\
\midrule
Shallow-layer reference & 1 & 5.10 & 163.75 & 0.54 \\
Geometry-selected & 6 & \textbf{4.96} & \textbf{142.62} & \textbf{0.42} \\
\bottomrule
\end{tabular}
\end{table}

\section{Baselines and Generation Metrics}
\label{app:baseline-details}
 
\subsection{Parameter Accounting}
\label{app:param-accounting}
 
Table~\ref{tab:baseline_param_counts} reports trainable and active module parameters for the continuous-diffusion comparison in Section~\ref{sec:continuous-diffusion-baselines}. Trainable parameters are updated during bridge or recovery training. Active module parameters include both trainable and frozen parameters used within that module, such as the frozen VAE in DiHAL. They exclude retained pretrained components outside the module.

\begin{table}[t]
\centering
\small
\caption{Trainable and active diffusion/recovery module parameters. CoDAR counts are estimated from its MDLM-style diffusion backbone and GPT-2-small-style autoregressive decoder with cross-attention under the Qwen tokenizer vocabulary.}
\label{tab:baseline_param_counts}
\begin{tabular}{lrr}
\toprule
Method & Trainable params & Active module params \\
\midrule
DiHAL (Qwen, full) & 862.7M & 946.3M \\
DiHAL-small (Qwen) & 257.4M & 341.1M \\
DiHAL-small (OLMo) & $\sim$255.9M & 339.5M \\
Diffusion-LM & 93.7M & 93.7M \\
SED & 73.7M & 73.7M \\
LD4LG & 25.6M AE + 188M diffusion & 25.6M AE + 188M diffusion \\
CoDAR & 130M diffusion + 230.9M AR decoder & 360.9M \\
\bottomrule
\end{tabular}
\end{table}

For DiHAL and CoDAR, Table~\ref{tab:total_system_params} additionally reports total active system capacity. For DiHAL, this includes the input embedding, retained Transformer suffix, LM head, and all active bridge parameters. The replaced prefix is excluded. In OLMo, replacing layers~0--6 removes seven decoder blocks (469.8M parameters), retaining nine blocks (604.1M), the LM head (205.5M), and the input embedding (205.5M). Thus, DiHAL-small on OLMo uses approximately 1.35B active system parameters, compared with 7.95B on Qwen.

\begin{table}[t]
\centering
\small
\caption{System-level parameter accounting. Retained parameters include input embeddings, suffix, and LM head. Total active counts are approximate. The CoDAR count is an estimate.}
\label{tab:total_system_params}
\begin{tabular}{lrrr}
\toprule
Method & Active module & Retained & Total active \\
\midrule
DiHAL (Qwen, full) & 946.3M & 7.61B & 8.56B \\
DiHAL-small (Qwen) & 341.1M & 7.61B & 7.95B \\
DiHAL-small (OLMo) & 339.5M & 1.0151B & 1.35B \\
CoDAR & 360.9M & -- & 360.9M \\
\bottomrule
\end{tabular}
\end{table}

\subsection{Generation Metrics}
\label{app:generation-metrics}
 
Generative perplexity (Gen.PPL) is the perplexity that a pretrained GPT-2 evaluator assigns to generated samples.
Diversity is defined as
\[
\mathrm{Div}=\prod_{n=2}^{4}\frac{|\text{unique } n\text{-grams}|}{|\text{total } n\text{-grams}|},
\]
the product of Distinct-2, Distinct-3, and Distinct-4.
We follow the unconditional-generation protocol of LD4LG, also adopted by CoDAR, so that results are comparable with prior continuous diffusion language models.
A GPT-2 evaluator can favor text close to the GPT-2 distribution and can reward predictable or repetitive outputs, so neither metric alone establishes overall generation quality across model families.
 
\subsection{Bridge-Capacity and Initialization Control}
\label{app:capacity-controls}
 
To examine sensitivity to bridge capacity and initialization, we construct a smaller SD-style U-Net by reducing \texttt{layers\_per\_block} from 2 to 1 and using \texttt{block\_out\_channels}$=(224,448,832,832)$. The U-Net is randomly initialized; the pretrained VAE remains frozen. On Qwen, the U-Net has 254.3M parameters and the complete active bridge has 341.1M parameters, including the 83.7M frozen VAE and approximately 3.1M projection and conditioning parameters. On OLMo, the active bridge has 339.5M parameters.

\begin{table}[t]
\centering
\small
\caption{Bridge-capacity and initialization controls across Qwen3-8B and OLMo-2-1B. All systems use matched diffusion/recovery training budgets, not matched total system capacity.}
\label{tab:small_bridge_control}
\resizebox{\linewidth}{!}{%
\begin{tabular}{llrrr}
\toprule
Method & U-Net initialization & Active module & Gen.PPL $\downarrow$ & Div. $\uparrow$ \\
\midrule
DiHAL (Qwen, full) & Pretrained SD-v1.5 & 946.3M & \textbf{136.02} & \textbf{0.5913} \\
DiHAL-small (Qwen) & Random & 341.1M & 140.51 & 0.5230 \\
DiHAL-small (OLMo) & Random & 339.5M & 138.32 & 0.5812 \\
CoDAR & Method-specific & 360.9M & 144.83 & 0.4777 \\
\bottomrule
\end{tabular}%
}
\end{table}

Reducing the Qwen bridge size by approximately $64\%$ and removing pretrained U-Net initialization still yields better Gen.PPL and diversity than CoDAR. The OLMo variant further reduces the active system size while improving both generation metrics relative to the small Qwen variant. This supports the feasibility of the hidden-state interface with smaller bridges and backbones, while the comparisons do not isolate interface choice independently of total system capacity or inference FLOPs.

\subsection{Baseline Implementations}
\label{app:baseline-implementations}

The baselines in Section~\ref{sec:continuous-diffusion-baselines} are independently implemented and trained following their respective approaches rather than integrated into Qwen or OLMo. All methods use the same 300K-example corpus and a matched diffusion/recovery training budget of 40 H100 GPU-hours. They differ in the representation being denoised and the mechanism used to produce discrete tokens, summarized in Table~\ref{tab:baseline_interfaces}.

\begin{table}[t]
\centering
\small
\caption{Denoising spaces and token recovery mechanisms of the compared systems.}
\label{tab:baseline_interfaces}
\resizebox{0.96\linewidth}{!}{%
\begin{tabular}{lll}
\toprule
Method & Diffusion target & Token recovery \\
\midrule
DiHAL & Transformer hidden state & Retained suffix + LM head \\
Diffusion-LM & Token embeddings & Embedding-to-token mapping \\
SED & Self-conditioned embeddings & Embedding-to-token mapping \\
LD4LG & Learned text latents & Pretrained BART-based decoding \\
CoDAR & Continuous token/latent states & Contextual AR decoder \\
\bottomrule
\end{tabular}%
}
\end{table}

\section{Inference Cost}
\label{sec:inference-cost}
 
\subsection{Cost across Insertion Depths}
\label{app:cost-depth}
 
The end-to-end inference cost of DiHAL depends on the insertion depth and on the number of denoiser evaluations (NFEs).
Table~\ref{tab:efficiency} reports latency, throughput, and peak memory across insertion layers, bridge denoising, Transformer layers that are kept, and the LM head are reported as well.
 
At the geometry-selected shallow layers, DiHAL is slower than the original backbone even with a single denoising step, and latency grows substantially with the number of NFEs, which indicates that bridge denoising dominates the overhead.
Deeper insertion layers can reduce latency at one NFE by skipping more Transformer layers, but these layers are not selected by the geometry criterion and are much harder to reconstruct.
Peak memory increases at shallow insertion layers but decreases at deeper layers as more Transformer blocks are removed.
The current implementation should therefore not be interpreted as an acceleration method: deeper replacement removes more Transformer computation, but denoising overhead and reconstruction difficulty prevent practical speedups in the present setting.
 
\begin{table}[t]
\centering
\scriptsize
\setlength{\tabcolsep}{3pt}
\caption{End-to-end inference cost. Latency is in milliseconds per token, throughput in tokens per second, and peak memory in GB. DiHAL cost includes bridge denoising, the retained Transformer layers, and the LM head.}
\label{tab:efficiency}
\begin{tabular}{llccccc}
\toprule
Model & Method & Insert. layer & NFEs & Latency $\downarrow$ & Throughput $\uparrow$ & Peak mem. $\downarrow$ \\
\midrule
\midrule
Llama-3.1-8B & Original & -- & -- & 0.043 & 23,509 & 16.2 \\
Llama-3.1-8B & DiHAL & 3  & 1  & 0.074 & 13,481 & 16.8 \\
Llama-3.1-8B & DiHAL & 3  & 4  & 0.150 & 6,681  & 16.8 \\
Llama-3.1-8B & DiHAL & 3  & 20 & 0.503 & 1,989  & 16.8 \\
Llama-3.1-8B & DiHAL & 10 & 1  & 0.065 & 15,314 & 13.9 \\
Llama-3.1-8B & DiHAL & 10 & 4  & 0.146 & 6,844  & 13.9 \\
Llama-3.1-8B & DiHAL & 10 & 20 & 0.515 & 1,941  & 13.9 \\
Llama-3.1-8B & DiHAL & 20 & 1  & 0.051 & 19,467 & 9.9 \\
Llama-3.1-8B & DiHAL & 20 & 4  & 0.116 & 8,597  & 9.9 \\
Llama-3.1-8B & DiHAL & 20 & 20 & 0.503 & 1,988  & 9.9 \\
Llama-3.1-8B & DiHAL & 30 & 1  & 0.035 & 28,676 & 5.8 \\
Llama-3.1-8B & DiHAL & 30 & 4  & 0.114 & 8,798  & 5.8 \\
Llama-3.1-8B & DiHAL & 30 & 20 & 0.502 & 1,993  & 5.8 \\
\midrule
Qwen3-8B & Original & -- & -- & 0.059 & 17,033 & 15.9 \\
Qwen3-8B & DiHAL & 2  & 1  & 0.081 & 12,375 & 16.9 \\
Qwen3-8B & DiHAL & 2  & 4  & 0.130 & 7,668  & 16.9 \\
Qwen3-8B & DiHAL & 2  & 20 & 0.503 & 1,989  & 16.9 \\
Qwen3-8B & DiHAL & 12 & 1  & 0.066 & 15,222 & 13.3 \\
Qwen3-8B & DiHAL & 12 & 4  & 0.132 & 7,554  & 13.3 \\
Qwen3-8B & DiHAL & 12 & 20 & 0.506 & 1,977  & 13.3 \\
Qwen3-8B & DiHAL & 22 & 1  & 0.047 & 21,156 & 9.7 \\
Qwen3-8B & DiHAL & 22 & 4  & 0.119 & 8,428  & 9.7 \\
Qwen3-8B & DiHAL & 22 & 20 & 0.500 & 2,002  & 9.7 \\
Qwen3-8B & DiHAL & 32 & 1  & 0.036 & 27,581 & 6.1 \\
Qwen3-8B & DiHAL & 32 & 4  & 0.093 & 10,769 & 6.1 \\
Qwen3-8B & DiHAL & 32 & 20 & 0.500 & 2,001  & 6.1 \\
\bottomrule
\end{tabular}
\end{table}
 
\subsection{NFE--Quality--Throughput Trade-off}
\label{app:nfe-tradeoff}
 
Inference cost alone does not show how many denoising evaluations are needed for usable reconstruction.
Table~\ref{tab:nfe_quality_tradeoff} therefore reports output fidelity and inference cost jointly at the geometry-selected Llama layer~3 for $N_{\mathrm{FE}}\in\{1,4,20\}$.
 
\begin{table}[t]
\centering
\scriptsize
\setlength{\tabcolsep}{3pt}
\caption{Quality--efficiency trade-off across denoising budgets at the selected Llama-3.1-8B layer~3. The best DiHAL result is in bold and the original model is shown as a reference. The $N_{\mathrm{FE}}=20$ row corresponds to Table~\ref{tab:full_training}.}
\label{tab:nfe_quality_tradeoff}
\resizebox{\linewidth}{!}{%
\begin{tabular}{lrrrrrrrr}
\toprule
Model & NFE & KL $\downarrow$ & NLL $\downarrow$ & PPL $\downarrow$ & PPL ratio $\downarrow$ & Latency $\downarrow$ & Throughput $\uparrow$ & Rel. throughput $\uparrow$ \\
\midrule
Original (reference) & -- & 0.00 & 1.87 & 6.47 & $1.00\times$ & 0.043 & 23,509 & 100.0\% \\
\midrule
DiHAL, L3 & 1  & 0.73 & 4.89 & 133.20 & $20.59\times$ & \textbf{0.074} & \textbf{13,481} & \textbf{57.3\%} \\
DiHAL, L3 & 4  & \textbf{0.71} & \textbf{4.81} & \textbf{122.52} & $\mathbf{18.94\times}$ & 0.150 & 6,681 & 28.4\% \\
DiHAL, L3 & 20 & 0.73 & 4.91 & 135.64 & $20.96\times$ & 0.503 & 1,989 & 8.5\% \\
\bottomrule
\end{tabular}%
}
\end{table}
 
Increasing the denoising budget does not monotonically improve fidelity.
Four NFEs give the best KL, NLL, and perplexity among the evaluated settings, but perplexity remains $18.94\times$ that of the original model while throughput falls to $28.4\%$ of the original.
At 20 NFEs, throughput falls to $8.5\%$ and fidelity is worse than at four NFEs.
The current shallow DiHAL configuration therefore does not offer a favorable end-to-end quality--efficiency trade-off relative to its source model.
 
 
\section{Sensitivity of the Retained Suffix}
\label{app:suffix-sensitivity}
 
The final output depends both on boundary reconstruction error and on how the retained suffix responds to that error.
We analyze this sensitivity with the trained Llama-3.1-8B layer-3 bridge, without additional training.
Let $h_{\ell^\ast}$ be the original boundary representation and $\hat h_{\ell^\ast}$ the bridge reconstruction.
We interpolate along the observed error direction,
\[
h_\alpha=h_{\ell^\ast}+\alpha\bigl(\hat h_{\ell^\ast}-h_{\ell^\ast}\bigr),
\qquad
\alpha\in\{0,0.25,0.5,0.75,1,1.5,2\},
\]
where $\alpha=0$ recovers the original state, $\alpha=1$ is the actual bridge output, and $\alpha>1$ amplifies the error in the same direction.
Each $h_\alpha$ is passed through the same frozen suffix and LM head.
 
\begin{table}[t]
\centering
\small
\caption{Sensitivity of the retained Llama suffix to interpolation along the observed reconstruction-error direction. $\alpha=0$ corresponds to the original model and $\alpha=1$ to the DiHAL result in Table~\ref{tab:full_training}.}
\label{tab:suffix_sensitivity}
\begin{tabular}{rrrrr}
\toprule
$\alpha$ & Teacher KL $\downarrow$ & NLL $\downarrow$ & PPL $\downarrow$ & Top-token agreement $\uparrow$ \\
\midrule
0.00 & 0.00 & 1.87 & 6.47   & 1.0000 \\
0.25 & 0.21 & 3.07 & 21.52  & 0.7341 \\
0.50 & 0.51 & 4.12 & 61.39  & 0.6079 \\
0.75 & 0.70 & 4.77 & 118.21 & 0.6006 \\
1.00 & 0.73 & 4.91 & 135.64 & 0.5913 \\
1.50 & 0.77 & 5.05 & 155.72 & 0.5851 \\
2.00 & 0.86 & 5.36 & 212.52 & 0.5809 \\
\bottomrule
\end{tabular}
\end{table}
 
The suffix is sensitive to boundary perturbations: at $\alpha=0.25$, teacher KL already rises to $0.21$ and top-token agreement falls to $73.4\%$.
The response is nonetheless gradual rather than showing a sharp failure threshold, with most of the degradation occurring before $\alpha\approx0.75$.
At the actual operating point ($\alpha=1$), KL, NLL, and PPL are $0.73$, $4.91$, and $135.64$, and doubling the same error direction increases PPL to $212.52$.
 
We also track the discrepancy through the retained layers.
For retained layer $j$, let
\[
D_j=\frac{\bigl\|\hat h^{(j)}-h^{(j)}\bigr\|_2}{\bigl\|h^{(j)}\bigr\|_2},
\]
where $h^{(j)}$ and $\hat h^{(j)}$ are computed from the original and reconstructed boundary states.
The ratio $D_j/D_{\ell^\ast}$ increases gradually from $1.00\times$ at the insertion layer to about $1.44\times$ at the final block, so the suffix moderately amplifies reconstruction error without an abrupt layer-wise explosion.
 
 
\section{Existing Assets and Licenses}
\label{app:assets}

Our experiments use publicly available pretrained backbones, datasets, and model components.
We use Llama-3.1-8B-Instruct under the Llama 3.1 Community License, and Qwen3-8B and OLMo-2-0425-1B under the Apache 2.0 license.
We use Dolma v1.7 for activation extraction and bridge training under the ODC-BY license, and WikiText-103 for evaluation under its Creative Commons/GFDL licensing terms.
The diffusion bridge reuses Stable Diffusion v1.5 components under the CreativeML OpenRAIL-M license, and generative perplexity is computed with a GPT-2 evaluator under the GPT-2 license terms.
We cite the original papers and model or dataset sources for all existing assets and use them for research evaluation consistent with their respective terms.

\end{document}

%% file: math_commands.tex
\usepackage{amsmath,amsfonts,bm}

\def\eqref#1{equation~\ref{#1}}

\def\1{\bm{1}}

\DeclareMathAlphabet{\mathsfit}{\encodingdefault}{\sfdefault}{m}{sl}
\SetMathAlphabet{\mathsfit}{bold}{\encodingdefault}{\sfdefault}{bx}{n}

